\documentclass[accepted]{uai2022} 
\usepackage[american]{babel}

\usepackage{natbib} 
\usepackage{mathtools} 
\usepackage{booktabs} 
\usepackage{tikz} 

\usepackage[utf8x]{inputenc}
\usepackage{mathrsfs}
\usepackage{amssymb}
\usepackage{amsfonts}
\usepackage{amsmath}
\usepackage{amsthm,verbatim}
\usepackage{wrapfig}
\usepackage{multirow}
\usepackage{longtable}
\usepackage{enumerate}
\usepackage{color, graphicx}
\usepackage{tikz}

\newcommand{\bA}{{\boldsymbol A}}

\newcommand{\bE}{{\boldsymbol E}}
\newcommand{\bG}{{\boldsymbol G}}

\newcommand{\bM}{{\boldsymbol M}}
\newcommand{\bI}{{\boldsymbol I}}

\newcommand{\bQ}{{\boldsymbol Q}}

\newcommand{\bU}{{\boldsymbol U}}
\newcommand{\bV}{{\boldsymbol V}}

\newcommand{\bX}{{\boldsymbol X}}
\newcommand{\bY}{{\boldsymbol Y}}
\newcommand{\bZ}{{\boldsymbol Z}}

\newcommand{\bSigma}{{\boldsymbol \Sigma}}

\newcommand{\bOmega}{{\boldsymbol \Omega}}
\newcommand{\boldeta}{{\boldsymbol \eta}}

\newcommand{\bOnes}{{\boldsymbol 1}}
\newcommand{\bZeros}{{\boldsymbol 0}}

\newcommand{\ba}{{\boldsymbol a}}

\newcommand{\br}{{\boldsymbol r}}
\newcommand{\bs}{{\boldsymbol s}}

\newcommand{\bv}{{\boldsymbol v}}

\newcommand{\bx}{{\boldsymbol x}}

\newcommand{\bz}{{\boldsymbol z}}

\newcommand{\BR}{\mathbb{R}}
\newcommand{\BE}{\mathbb{E}}
\newcommand{\BI}{\mathbb{I}}
\newcommand{\BP}{\mathbb{P}}

\newcommand{\mA}{{\cal A}}

\newcommand{\mC}{{\cal C}}

\newcommand{\mF}{{\cal F}}

\newcommand{\mI}{{\cal I}}

\newcommand{\mS}{{\cal S}}

\newcommand{\balpha}{{\boldsymbol \alpha}}
\newcommand{\bbeta}{{\boldsymbol \beta}}

\newcommand{\btheta}{{\boldsymbol \theta}}
\newcommand{\bepsilon}{{\boldsymbol \epsilon}}

\newcommand{\bmu}{{\boldsymbol \mu}}
\newcommand{\bnu}{{\boldsymbol \nu}}

\def\boxit#1{\vbox{\hrule\hbox{\vrule\kern6pt
          \vbox{\kern6pt#1\kern6pt}\kern6pt\vrule}\hrule}}

\newtheorem{assumption}{Assumption}
\newtheorem{theorem}{Theorem}[section]
\newtheorem{lemma}{Lemma}[section]

\newtheorem{corollary}[theorem]{Corollary}

\newtheorem{remark}{Remark}

\usepackage{graphicx}
\usepackage{framed}
\usepackage[normalem]{ulem}
\usepackage{indentfirst}
\usepackage{amsmath,amsthm,amssymb,amsfonts}
\usepackage[italicdiff]{physics}
\usepackage[T1]{fontenc}
\usepackage{pifont} 
\usepackage{mathdots} 
\usepackage{wrapfig}
\usepackage{lmodern,mathrsfs}
\usepackage{tikz,tikz-3dplot,tikz-cd,tkz-tab,tkz-euclide,pgf,pgfplots}
\pgfplotsset{compat=newest}
\usepackage{multicol}

\usepackage{hyperref}
\usepackage[nameinlink]{cleveref} 

\usepackage{algorithm}
\usepackage{algpseudocode}

\title{Residual Bootstrap Exploration for Stochastic Linear Bandit}

\author[1]{Shuang Wu}
\author[1,2]{Chi-Hua Wang}
\author[1]{Yuantong Li}
\author[1]{Guang Cheng}
\affil[1]{%
    Department of Statistics\\
    University of California, Los Angeles\\
    Los Angeles, California, USA
}
\affil[2]{%
    Department of Statistics\\
    Purdue University\\
    West Lafayette, Indiana, USA
}
  
\begin{document}
\maketitle

\begin{abstract}
We propose a new bootstrap-based online algorithm for stochastic linear bandit problems. The key idea is to adopt residual bootstrap exploration, in which the agent estimates the next step reward by re-sampling the residuals of mean reward estimate. Our algorithm, residual bootstrap exploration for stochastic linear bandit (\texttt{LinReBoot}), estimates the linear reward from its re-sampling distribution and pulls the arm with the highest reward estimate. In particular, we contribute a theoretical framework to demystify residual bootstrap-based exploration mechanisms in stochastic linear bandit problems. The key insight is that the strength of bootstrap exploration is based on collaborated optimism between the online-learned model and the re-sampling distribution of residuals. Such observation enables us to show that the proposed \texttt{LinReBoot} secure a high-probability $\Tilde{O}(d \sqrt{n})$ sub-linear regret under mild conditions. Our experiments support the easy generalizability of the \texttt{ReBoot} principle in the various formulations of linear bandit problems and show the significant computational efficiency of \texttt{LinReBoot}. 
\end{abstract}

\section{Introduction}\label{sec:intro}

Stochastic linear bandit is an online learning problem that the learning agent acts by pulling arms, where each arm is associated with a feature vector, then learning the arms information from the corresponding random rewards. In such problems, the typical goal of a learning agent is to maximize its cumulative reward.
Learning more about an arm (explore) or pulling the arm with the highest estimated reward (exploit) leads to the well-known \textit{exploration- exploitation trade-off}, which is the central trade-off captured in many decision-making applications in modern online service industries. 
Consequently, the design of stochastic linear bandit algorithms demands an easy-generalizable implementation across various contextualize actions and reward generation processes.  

In the past decade of bandit literature, such demands have invited researchers to investigate bootstrap-based exploration-exploitation trade-offs and have drawn rising attention \citep{baransi2014sub, eckles2014thompson, osband2015bootstrapped, vaswani2018new, hao2019bootstrapping, kveton2019garbage, wang2020residual}. Yet, prior works on bootstrap-based bandit algorithms focus on provable multi-armed bandit algorithms and only provide a limited empirical evaluation of bootstrap-based stochastic linear bandit algorithms, and their theoretical counterpart remains unknown. Such knowledge gap of bootstrapping stochastic linear bandit persuades our investigation on the provable bootstrap-based stochastic linear bandits: \textbf{Can we theoretically and empirically support the validity and easy-generalizability of bootstrapping procedure in stochastic linear bandit algorithms design?} In particular, we aim to deliver a generic framework to demystify the bootstrap optimism in stochastic linear bandit problems and validate the easy generalizability of the bootstrap principle across various contextual linear bandit problems.

\textbf{Contributions.} 
We introduce \texttt{LinReBoot} algorithms that implement Residual Bootstrap Exploration for stochastic linear bandit problem with sub-linear regret. We theoretically show that \texttt{LinReBoot} secures $\Tilde{O}(d \sqrt{n})$ regret where $d$ is the dimension of features. This sub-linear regret bound matches the regret bound of the same order as those theoretical results of Linear Thompson Sampling algorithms. The key to achieving such sub-linear regret guarantee is to carefully manage and collaborate sample and bootstrap optimism (Section \ref{sec:colla_opti}). In particular, by measuring the ''sample-bootstrap optimistic estimated discrepancy ratio'' of the optimal arm, \texttt{LinReboot} successfully avoids over or under exploration and theoretically secures sub-linear mean regret with high-probability. To our knowledge, this is the first theoretical analysis to support the validity and efficiency of the residual bootstrap-based procedure for stochastic linear bandit problems. We empirically show that \texttt{LinReBoot} rivals or exceeds competing algorithms including Linear Thompson Sampling, Linear PHE, Linear GIRO, and Linear UCB under stochastic linear bandit problem as well as more complicated linear bandit settings. These significant results support the easy-generalizability of proposed \texttt{LinReBoot}. In summary, our contributions are as follows:
\begin{itemize}[leftmargin=5pt, itemsep = -2pt]
\item Propose \texttt{LinReBoot} algorithms that implement Residual Bootstrap Exploration in linear bandit problems without boundness assumption of rewards.
\item  Theoretically show that \texttt{LinReBoot} secures $\Tilde{O}(d \sqrt{n})$ regret, matching the regret bound of the same order as those theoretical results of Linear Thompson Sampling algorithms.
\item Empirically show that \texttt{LinReBoot} rivals or exceeds baseline algorithms and supports that \texttt{LinReBoot} is easy-generalizable among linear bandit problems.
\end{itemize}

\textbf{Related Works.} 
Bootstrap-based contextual bandit algorithms design has been actively studied in the last half-decade and drawn a surge of interest from both theoretical studies and industrial practice \citep{elmachtoub2017practical, eckles2014thompson, osband2016deep, kveton2019garbage, hao2019bootstrapping}. Bootstrap-based bandit algorithm design is a paradigm of sequential decision-making based on an exploration mechanism with no pre-defined mean reward model. Such paradigm enjoys a decisive advantage that engineers are free to deploy any reward model of interests without painful adaption to problem structure \citep{kveton2019garbage, kveton2019perturbedMAB}. \texttt{ReBoot} \citep{wang2020residual} provided a theoretical logarithmic regret guarantee for multi-armed bandit (MAB) and empirical investigation to validate the easy generalizability of the \texttt{ReBoot} principle. Our work aims to provide a theoretical guarantee for the bootstrap-based linear bandit algorithms and empirically investigate more general contextual linear bandit setting to validate the \texttt{ReBoot} principle.

One close related work is \citep{kveton2019perturbed} which introduces perturbation of past samples for exploration under stochastic linear bandit problem. The limitation of \citep{kveton2019perturbed} is the boundness of rewards, indicating many broader classes of rewards such as Gaussian rewards are not applicable with a theoretical guarantee. In contrast, the proposed \texttt{LinReBoot} algorithms relax the boundness reward assumption and thus validate bootstrap-based bandit algorithms in wider bandit environments with a broader class of reward generation processes. 

Early works about exploration in bandit problems \citep{abbasi2011improved, langford2007epoch, dani2008stochastic} are practical but no guarantee of the optimality. Some works \citep{wang2020residual, kveton2019garbage, kveton2019perturbedMAB, thompson1933likelihood, auer2002finite} provide well designed exploration for bandit problems and have their own principles for adopting to more general problems. In these works, three principles including \texttt{ReBoot}\citep{wang2020residual}, \texttt{GIRO}\citep{kveton2019garbage} and \texttt{PHE}\citep{kveton2019perturbedMAB} are devising exploration mechanism based on up-to-now history instead of on pre-defined reward model in the other two principles \texttt{TS}\citep{thompson1933likelihood} and  \texttt{UCB}\citep{auer2002finite}. Our work generalizes \texttt{ReBoot} into stochastic linear bandit problems.

\textbf{Notations.} 
Let $[n]$ be set $\{1,2,...,n\}$. $\bOnes$ is a vector with all ones and $\bI$ is the identity matrix. For a vector $\bv$, $\norm{\bv}_{2}$ is $2$-norm of $\bv$ and $\norm{\bv}_{\bA}^2:=\sqrt{ \bv^{\top} \bA \bv}$ for a semidefinite matrix $\bA$. Let $\langle \cdot, \cdot \rangle$ be the inner product operation. Denote $\mF_{t}$ as the history of randomness up to round $t$. $\BE_{t}[\cdot]:= \BE[\cdot|\mF_{t-1}]$ is defined as the conditional expectation given $\mF_{t-1}$ and $\BP_{t}(\cdot):= \BP(\cdot|\mF_{t-1})$ is defined as the conditional probability given $\mF_{t-1}$. $\BI\{\cdot\}$ is indicator function. For a set or event $E$, we denote its complement as $\Bar{E}$. $N(\mu, \sigma^2)$ is Gaussian distribution with mean $\mu$ and variance $\sigma^2$.   We use $\Tilde{O}$ for big $O$ notation up to logarithmic factor.

\section{Stochastic Linear Bandit} \label{sec: Stochastic_Linear_Bandit}

\textbf{Contextualize Action Set.} In stochastic linear bandit problem, we identify the actions with $d-$dimensional features from $\mA \subset \BR^{d}$ and assume $|\mA|$, the size of the action set, is finite. Let $K:=|\mA|$ be the number of actions (arms), $\bx_{k} \in \BR^{d}$ be the context vector of the $k$-th arm, that is, $\mA = \{\bx_1,...,\bx_K\}$. 

\textbf{Reward generating mechanism.} The reward function is parameterized by $\btheta \in \BR^{d}$ such that, at time $t$ the agent chooses an action $I_{t} \in [K]$ with feature $X_{t} = \bx_{I_t} \in \mA$, the reward is generated by
\begin{equation}\label{eq:reward_model}
Y_{t} \equiv \langle X_{t}, \btheta \rangle + \epsilon_{t} . 
\end{equation}
Specifically, the reward obtained by the agent at round $t$ when pulling arm $I_t = k$ is generated from a distribution with mean $\mu_{k}:=\bx_{k}^{\top} \btheta$, conditioning on context $\bx_{k}$. The property of noise $\epsilon_{t}$ is described in Assumption \ref{ass:noise_bound}. Furthermore,  denote the recieved reward by $r_{I_{t}}$ and the reward random variable by $Y_{t}$ at round $t$.

\textbf{Regret.} Without loss of generality, assume that arm $1$ is the unique optimal arm, that is $\mu_1 > \mu_k \mbox{ } \forall k \neq 1$. The optimal gap of the $k$-th arm is $\Delta_{k}: = \mu_1 - \mu_k \geq 0$. The expected $n$-round regret is denoted as 
\begin{equation}\label{eq: expected_regret}
R_{n} := \sum_{k=2}^{K} \Delta_{k} \BE[\sum_{t=1}^{n} \BI\{I_{t} = k\}] .
\end{equation}
The goal of the agent is to maximize the expected cumulative reward in $n$ rounds, which is equivalent to minimizing the expected regret $R_{n}$.

\begin{assumption}\label{ass:bound}(Boundness assumptions) 
True parameter $\btheta$ is bounded: $\norm{\btheta}_2 \leq S_2$.
\end{assumption}
Besides, we denote $L$ as the upper bound for context vectors: $\norm{\bx_{k}}_2 \leq L$ for all $k \in [K]$. 
Assumption \ref{ass:bound} is referred to the boundness assumptions in the stochastic linear bandit literature and is to ensure the regret is bounded if the agent pulls any sub-optimal actions (see Section 5 in \citep{abbasi2011improved}).
\begin{assumption}(Noise Clipping assumption)\label{ass:noise_bound} Noise process $\{\epsilon_{t}\}_{t=1}^{\infty}$ described in \eqref{eq:reward_model} satisfies that for some $L_{1}, L_{2} > 0$,
\begin{equation}
\begin{aligned}
    e^{L_{1} \eta^2} \leq \BE[e^{\eta \epsilon_{t}} | \mF_{t-1}] \leq e^{L_{2} \eta^2} \mbox{, } \forall \eta \geq 0,
\end{aligned}
\end{equation}
where $\mF_{t-1}=\{\epsilon_1, I_{1}, \cdots, \epsilon_{t-1}, I_{t-1}\}$.
\end{assumption}
Assumption \ref{ass:noise_bound} implies that stochastic process $\{\epsilon_{t}\}_{t=1}^{\infty}$ is conditionally sub-gaussian with constant $L_{2}$. $L_{1}$ contributes to the lower bound of moment generating function suggested by \citep{zhang2020non}. Note that the Assumption \ref{ass:noise_bound} allows heteroscedasticity among different arms by choosing $L_{2}$ as the largest variance among arms. Such heteroscedasticity consideration arises and has been identified as a challenge in applications of Bayesian optimization \citep{kirschner2021information, cowen2020empirical}.

\section{Residual Bootstrap Exploration}
\label{sec: LinReBoot_alg}


\subsection{\texttt{ReBoot} Principle}
\label{subsec: ReBoot_principle}
This section presents essential proof of concepts to implement \texttt{ReBoot} principle \citep{wang2020residual}. In general, each round of interaction, the decision policy admits four subroutines to implement \texttt{ReBoot} principle: 1) Learning, 2) Fitting, 3) Bootstrapping, and 4) Exploring. Following elaborates on each subroutine:

\textbf{1) Model Learning.} The first subroutine outputs a learned model based on current collected data. Our implementation learns the parameter $\btheta$ in Eq.\eqref{eq:reward_model} by some user-specified model.

\textbf{2) Data Fitting.} The second subroutine fits the current data set with the learned model in the previous subroutine and then outputs the residual set.
Intuitively, the residuals measure the \textit{goodness of fit} of the learned model and should drop a hint on the right amount of exploration. In other words, the residuals should suggest a right magnitude of exploration bonus in decision policy \eqref{eq:policy}. How to manage and integrate uncertainty behind residuals into the exploration mechanism of policy is the main challenge.

\textbf{\textbf{3)} Residuals Bootstraping.} The third subroutine associates the residuals obtained the last subroutine with a bootstrapping distribution. Instead of maintaining a belief distribution on a parameter in the Bayesian approach, \texttt{ReBoot} principle maintains a bootstrapping distribution on the statistical error based on residuals.
The challenge is to justify the efficacy of residual-based optimism construction in both theory and practice.

\textbf{4) Actions Exploring.} The fourth subroutines sample the exploration bonus from the bootstrapping distribution and output an index for each action. Such bootstrap procedure is more computationally efficient than prior efforts since this procedure only requires drawing a sample from the bootstrapping distribution. The challenge is to prove that such bootstrap procedure secures sub-linear regret in theory.

\subsection{\texttt{LinReBoot} Algorithm}
\label{subsec: LinReBoot}
We propose the Linear Residual Bootstrap Exploration algorithm (\texttt{LinReBoot}, Algorithm \ref{alg:LinReBoot: Version_1}) for stochastic linear bandit problems. 
This section elaborates the four subroutines in Section \ref{subsec: ReBoot_principle} for the proposed \texttt{LinReBoot}. 

\textbf{1)}
\texttt{LinReBoot} uses ridge regression procedure, whose learned parameter is $\Hat{\btheta}_{t} $ \eqref{eq:fitted_theta} and estimated mean reward for arm $k$ is $\Hat{\mu}_{k,t}$ \eqref{eq:fitted_mean}. Such way to estimate mean reward is easy to manage the confidence \citep{abbasi2011improved}. Thus, we focus on confidence management for the bootstrap-based exploration. 



\textbf{Ridge Regression Procedure.} \texttt{LinReBoot} fits linear model at round $t$ as follow,
\begin{subequations}
\label{def: LSE}
\begin{align}
    \bV_{t} 
    &= \bX_{t-1}^{\top} \bX_{t-1} + \lambda \bI, \\
    \Hat{\btheta}_{t} 
    &= \bV_{t}^{-1} \bX_{t-1}^{\top} \bY_{t-1}\label{eq:fitted_theta}, \\
    \Hat{\mu}_{k,t}
    &= \bx_{k}^{\top} \Hat{\btheta}_{t} \mbox{, } \forall k \in [K], 
\label{eq:fitted_mean}
\end{align}
\end{subequations}
where $\bX_{t-1} = (X_{1},...,X_{t-1})^{\top} \in \BR^{(t-1) \times d}$. The $\tau$-th row of $\bX_{t-1}$ is the context $X_{\tau}^{\top}$ for $ \tau \in [t-1]$, $\bY_{t-1} = (Y_{1},...,Y_{t-1})^{\top}$ is reward vector whose elements are rewards up to round $t-1$. $\lambda$ denotes the regularization level. $\bV_{t}$ denotes the sample covariance matrix up to round $t$ and $\Hat{\btheta}_{t}$ is the ridge estimation of target parameter $\btheta$ in \eqref{eq:reward_model}. $\Hat{\mu}_{k,t}$ denotes the estimated mean of arm $k$ based on history. Note that the first $K$ rounds in proposed \texttt{LinReBoot} is fully exploring each arm once. In other words, $I_t = t$ when $t \in [K]$, indicating $\bX_{K}:=(\bx_1,...,\bx_K)^{\top} \in \BR^{K \times d}$. We call this $\bX_{K}$ the context matrix with rank $r \leq \min(K,d)$ and singular values $\sigma_{1}, ... , \sigma_{r}$. Also define $\sigma_{\min}^2 \leq \sigma_{i}^2 \leq \sigma_{\max}^2 \mbox{, } \forall i \in [r]$. With these definitions, we make a mild assumption about the shrinkage effect of ridge regression:

\begin{assumption} (Validity of Ridge Regression)
\label{ass: lower_bound}
The singular value decomposition of context matrix $\bX_{K}$ is denoted as $\bX_{K} :=\bG \bSigma \bU$ where $\bG \in \BR^{K \times K}$, $\bSigma \in \BR^{K \times d}$ and $\bU \in \BR^{d \times d}$. Define $\bOmega :=\bSigma (\bSigma^{\top} \bSigma + \lambda \bI)^{-1} \bSigma^{\top} \in \BR^{K \times K}$ and $\bZ := \bG \bOmega \bSigma \bU \in \BR^{K \times d}$. Let $\bz_1 \in \BR^{d}$ be the first row of $\bZ$. Given any $\lambda > 0$, there exists a corresponding  positive scalar $S_1$ such that $|\bx_1^{\top} \btheta - \bz_1^{\top} \btheta| \geq S_1$ for the $\theta$ in \eqref{eq:reward_model}.
\end{assumption}

\begin{remark}
Assumption \ref{ass: lower_bound} provides a lower bound of the absolute difference between true mean $\bx_1^{\top} \btheta$ and normalized mean $\bz_1^{\top} \btheta$ of the optimal arm. Note that if $\lambda \rightarrow 0$, then $\bz_1 \rightarrow \bx_1$ and $S_1 \rightarrow 0$. 
Thus this scalar $S_1$ measures the small perturbation on the mean of the optimal arm when the ridge regression procedure is applied. This $\bZ$ can be interpreted as a ridge shrinkage context matrix \citep{goldstein1974ridge}. One important phenomenon of online ridge regression is that even if the ridge estimator is biased, the shrinkage effect from ridge estimation provides exploration for the agent leading to making a correct decision. The positive scalar $S_1$ describes the shrinkage effect on the context. That is, the existence of $S_1$ indicates the ridge procedure is valid and its shrinkage effect exists. 
\end{remark}

\textbf{2)}
The fitting part of \texttt{LinReBoot} outputs the residuals under the linear model framework,
\begin{equation}
\begin{aligned}
\label{eq: residuals}
    e_{k,t,i} 
    &= r_{k,i} - \Hat{\mu}_{k,t} \mbox{, } \forall i \in [s_{k,t-1}],
\end{aligned}
\end{equation}
where $s_{k,t-1}:=\sum_{\tau=1}^{t-1} \BI\{I_{\tau} = k\}$ is the number of times pulling arm $k$ by round $t-1$, $r_{k,i}$ is the $i$-th reward of arm $k$ by round $t-1$. The \textit{goodness of fit} of the learned ridge regression model can be summarised by Residual Sum of Squares(RSS) \citep{archdeacon1994correlation} which is defined as
\begin{equation}
\begin{aligned}
    RSS_{k,t} := \sum_{i=1}^{s_{k,t-1}} e_{k,t,i}^2.
\end{aligned}
\end{equation}
Such measure plays an important role in the residual bootstrap exploration mechanism. 

\textbf{3)} The third part is Residuals Bootstrapping. This subroutine is independent of the model which suggests the power of generalizability of \texttt{ReBoot} principle. \texttt{ReBoot} principle requires the computation of the exploration bonus \citep{mammen1993bootstrap}, which is $s_{k,t-1}^{-1}\sum_{i=1}^{s_{k,t-1}} \omega_{k,t,i} e_{k,t,i}$, where $\{\omega_{k,t,i}\}_{i=1}^{s_{k,t-1}}$ is residual bootstrap weights for arm $k$ at round $t$.

\begin{algorithm}[t!]
\caption{\texttt{LinReBoot}}\label{alg:LinReBoot: Version_1}
\begin{algorithmic}
\Require $\lambda$, $s_{1,0}=...=s_{K,0}=0$
\For{$t = 1,...,n$}
    \If{$t < K + 1$}
        \State $I_{t} \leftarrow t$
    \Else
        \State $\bV_{t} \leftarrow \bX_{t-1}^{\top} \bX_{t-1} + \lambda \bI$
        \State $\Hat{\btheta}_{t} \leftarrow \bV_{t}^{-1} \bX_{t-1}^{\top} \bY_{t-1}$
        \For{$k=1,...,K$}
            \State $e_{k,t,i} \leftarrow r_{k,i} - \bx_{k}^{\top} \Hat{\btheta}_{t}$, $\forall i \in \{s_{k,t-1}\}$
            \State Generate $\{\omega_{k,t,i}\}_{i=1}^{s_{k,t-1}}$
            \State $\Tilde{\mu}_{k} \leftarrow \bx_{k}^{\top} \Hat{\btheta}_{t} +  
                    s_{k,t-1}^{-1}\sum_{i=1}^{s_{k,t-1}} \omega_{k,t,i} e_{k,t,i} $
        \EndFor
        \State $I_{t} \leftarrow  \underset{k \in [K]}{\arg\max} \mbox{ }  \Tilde{\mu}_{k}$
    \EndIf
    \State $s_{I_{t},t} \leftarrow s_{I_{t},t-1} + 1$ and $s_{k,t} \leftarrow s_{k,t-1}$. $\forall k \neq I_{t}$
    \State Pull arm $I_{t}$ and get reward $r_{I_{t}, s_{I_{t}}}$
    \State 
    $\bX_{t} \leftarrow
    \begin{bmatrix}
    \bX_{t-1} \\
    \bx_{I_{t}}^{\top}
    \end{bmatrix}$
    and 
    $\bY_{t} \leftarrow
    \begin{bmatrix}
    \bY_{t-1} \\
    r_{I_{t}, s_{I_{t}}}
    \end{bmatrix}$
\EndFor
\end{algorithmic}
\end{algorithm}

\textbf{Choice of Bootstrapping Weights.} The bootstrap weights considered in this work are i.i.d with zero mean and variance $\sigma_{\omega}^{2}$. They are independent of the noise process $\{\epsilon_{t}\}_{t=1}^{\infty}$. In the literature of bootstrap procedure \citep{mammen1993bootstrap} , the choices of bootstrap weights distribution include Gaussian weights, Rademacher weights and skew correcting weights. In \texttt{LinReBoot}, we adopt the Gaussian bootstrap weights to enable an efficient implement described at section \ref{subsec: Efficient_Implementation}.

\textbf{4)} The last subroutine is the action exploring based on residual bootstrap. More specifically, for arm $k$ at round $t$, \texttt{LinReBoot} adds exploration bonus from residual bootstrapping on the estimated mean $\Hat{\mu}_{k,t}$ as follow,
\begin{equation}
\label{eq:reboot_index}
    \Tilde{\mu}_{k,t}=\Hat{\mu}_{k,t} + \frac{1}{s_{k,t-1}} \sum_{i=1}^{s_{k,t-1}} \omega_{k,t,i} e_{k,t,i}, 
\end{equation}
then agent pulls arm with the highest bootstrapped mean,
\begin{equation}\label{eq:policy}
    I_{t} \equiv \arg\max_{k \in [K]} \mbox{ } \tilde{\mu}_{k,t}.
\end{equation}
Note that the variance of bootstrapped mean $\Tilde{\mu}_{k,t}$ is $\sigma_{\omega}^2 s_{k,t-1}^{-2}RSS_{k,t}$, indicating an adaptive amount of extra exploration is controlled by $s_{k,t-1}$ and $RSS_{k,t}$.

\textbf{Short Summary.}
Our proposed \texttt{LinReBoot} has following steps at round $t>K$,
\begin{itemize}[leftmargin=15pt, itemsep = -2pt]
    \item[\textbf{1)}] Ridge estimation: compute $\bV_{t}$, $\Hat{\btheta}_{t}$.
    \item[\textbf{2)}] Finding residuals for each arm: for arm $k$, compute $\Hat{\mu}_{k,t}$ and $\{e_{k,t,i}\}_{i=1}^{s_{k,t-1}}$.
    \item[\textbf{3)}] Compute Bootstrapped mean for each arm: for arm $k$, generate $\{\omega_{k,t,i}\}_{i=1}^{s_{k,t-1}}$ and compute $\Tilde{\mu}_{k,t}$ \eqref{eq:reboot_index}.
    \item[\textbf{4)}] Pull arm with the highest $\Tilde{\mu}_{k,t}$ then observe reward.
\end{itemize}

Algorithm \ref{alg:LinReBoot: Version_1} describes \texttt{LinReBoot}. The strength of \texttt{LinReBoot} is its easy generalizability across different bandit problems including linear bandits and even more complicated structured problems (Appendix \ref{appendix: experiment_algs}). 

\begin{remark} (\texttt{LinTS} perturbs system parameter estimate, \texttt{LinReBoot} perturbs expected reward estimates)
Compare with the \texttt{LinTS} in \citep{agrawal2013thompson}, in which \texttt{LinTS} samples a perturbed parameter $\tilde{\btheta}_{t}^{\texttt{LinTS}}=\hat{\btheta}_{t}+\beta_{t} \bV_{t}^{-1/2} \boldeta_{t}$ with scaling $\beta_{t}$ and appropriate independent noise $\boldeta_{t}$ (defined in \citep{agrawal2013thompson}). Our proposed \texttt{LinReBoot} samples a perturbed expected reward $
\tilde{\mu}_{k,t}^{\texttt{LinReBoot}} =\langle \hat{\btheta}_{t}, \bx_{k} \rangle + \frac{1}{s_{k,t-1}}\sum_{i=1}^{s_{k,t-1}}w_{k,t,i}e_{k,t,i}.$ That is, \texttt{LinReBoot} is perturbing the expected reward estimate via prediction error uncertainty, which is supervised by real reward. 
In contrast, \texttt{LinTS} is perturbing the system parameter, when can be wrong if the system modeling is wrong.
\end{remark}

\subsection{Efficient Implementation}
\label{subsec: Efficient_Implementation}

By the attractive computational properties of Gaussian distribution, the computational cost of \texttt{LinReBoot} can be reduced significantly when Gaussian Bootstrap weights are generated. Formally: assume $\omega_{k,t,i} \sim N(0, \sigma_{\omega}^{2}) \mbox{, } \forall k, t, i$, recalling \eqref{eq:reboot_index}, for $k\in[K]$ and any $t \geq 1$, bootstrapped mean $\Tilde{\mu}_{k,t}$ follows a Gaussian distribution,
\begin{equation}
\begin{aligned}\label{eq:conditional_distribution_bootstrap_index}
    \Tilde{\mu}_{k,t} | \mF_{t-1} \sim N(\Hat{\mu}_{k,t}, \sigma_{\omega}^2 s_{k,t-1}^{-2}RSS_{k,t}).
\end{aligned}
\end{equation}
Such Gaussian-distributed property of $\Tilde{\mu}_{k,t}$ indicates that if we can update $\Hat{\mu}_{k,t}$, $s_{k,t-1}$ and $RSS_{k,t}$ incrementally for arm $k$, this bootstrapped mean $\Tilde{\mu}_{k,t}$ can be generated by Gaussian generator without inner loop for generating weights. The first two terms, $\Hat{\mu}_{k,t}$ and $s_{k,t-1}$, are naturally updated in incremental manner. For $RSS_{k,t}$, following decomposition ensures an incremental update,
\begin{equation}
\begin{aligned}
\nonumber
    RSS_{k,t}  = \sum_{i=1}^{s_{k,t-1}} r_{k,i}^2 + s_{k,t-1} \Hat{\mu}_{k,t}^2 - 2 \Hat{\mu}_{k,t} \sum_{i=1}^{s_{k,t-1}} r_{k,i}.
\end{aligned}
\end{equation}
Then an efficient generation for $\Tilde{\mu}_{k,t} | \mF_{t-1}$ is ensured by the incremental updates for $\Hat{\mu}_{k,t}$, $s_{k,t-1}$, $\sum_{i=1}^{s_{k,t-1}} r_{k,i}^2$, $\sum_{i=1}^{s_{k,t-1}} r_{k,i}$. Furthermore, since the residual bootstrap weights are generated independently, $\Tilde{\mu}_{k,t}$ among arms are also independent given historical randomness and can be sampled from one multivariate Gaussian generation simultaneously. Formally, $\Tilde{\bmu}^{(t)} = (\Tilde{\mu}_{1,t},\dots, \Tilde{\mu}_{K,t})^\top$ is conditional distributed as
\begin{equation}
\begin{aligned}
    \Tilde{\bmu}^{(t)} | \mF_{t-1} \sim N_{K}(\Hat{\bmu}^{(t)}, \bSigma_{\omega}^{(t)}),
\end{aligned}
\end{equation}
where $\Hat{\bmu}^{(t)} = (\Hat{\mu}_{1,t}, \dots, \Hat{\mu}_{K,t})^\top$ and $\bSigma_{\omega}^{(t)}$ is a diagonal matrix with diagonal elements $\sigma_{\omega}^2 s_{k,t-1}^{-2}RSS_{k,t}$. Detailed steps and more illustration about efficient implementation is provided in Appendix \ref{appendix: experiment_efficient}. Moreover, an empirical study about computational efficiency is conducted in Appendix \ref{appendix: experiment_comp_cost} and Table.\ref{table: comp_cost} provides the computational cost of our proposed \texttt{LinReBoot} as well as other baseline algorithms.

\section{Optimism design}
\label{sec:TheoConsi}


\textbf{Optimistic Estimated Discrepancy.}
This section identifies and demystifies the technical challenge of implementing \texttt{ReBoot} principle in the stochastic linear bandit problem. The key is to conduct a detailed investigation to produce probabilistic control on the behavior of  the '\textbf{O}ptimistic \textbf{E}stimate \textbf{D}iscrepancy  (\textbf{OED})' of the \texttt{LinReBoot} policy \eqref{eq:policy}. In principle,  the \textbf{OED} is given by
\begin{equation}
\textbf{OED} =
\text{Optimism} \times \texttt{Action Context Norm},
\end{equation}
where the \texttt{Action Context Norm} is given by $\norm{\bx_{k}}_{\bV_{t}^{-1}}$ and \text{Optimism} is given by $c_{t,k}$ for the $k$th action at time $t$, defined in \eqref{eq:colla_optimism}. Design of $c_{t,k}$ will be elaborated in Section \ref{sec:colla_opti}.


\textbf{Sufficient Explored Arms.}
We define the concept of \textit{Sufficient Explore Arms} to facilitate the formal regret analysis of \texttt{LinReBoot}. Intuitively, an arm is \textit{sufficient explored} if its index produced by the policy \eqref{eq:policy} is less than the mean reward of the optimal arm. 
Technically, we say an arm $k$ is \textit{sufficiently explored} at time $t$ if the adopted OED ($c_{t,k}\norm{\bx_{k}}_{\bV_{t}^{-1}}$) is bounded by its optimal gap ($\Delta_{k}$). 

The above notion of sufficient explored arm defines the concept of ''set of sufficient explored arms'' $\mathcal{S}_{t}$, formally 
\begin{equation}\label{eq:SuffExpArms}
    \mS_{t} := \{k\in [K]: c_{t,k} \norm{\bx_{k}}_{\bV_{t}^{-1}} < \Delta_{k}\},
\end{equation}
where and $c_{t,k}$ is the collaborated optimism and $c_{t,k}\norm{\bx_{k}}_{\bV_{t}^{-1}}$ is an optimistic estimate of discrepancy of policy index \eqref{eq:policy}. 

The key consequence of set \eqref{eq:SuffExpArms} is that, any member in $\mS_{t}$ enjoys the property 
\begin{equation}
    \forall j \in \mS_{t} \cap [K] : \tilde{\mu}_{j,t} < \mu_{1};
\end{equation}
that is, the \texttt{LinReBoot} policy always avoids an index \eqref{eq:policy} from sufficiently explored subset such that the bootstrapped mean of this index is less than the optimal mean reward unless all arm are sufficiently explored. (see equation \eqref{appendix: proof_lemma_gap_a} in the proof of Lemma \ref{appendix: lemma_gap} at section \ref{appendix: proof_lemma_gap} for technical details).

\subsection{Collaborate  Optimism} \label{sec:colla_opti}

Here we elaborate on the collaborated optimism adopted in the definition of sufficient explored arms \eqref{eq:SuffExpArms}. 
Concretely, the collaborated optimism has a form
\begin{equation}
\label{eq:colla_optimism}
c_{t,k} = c_1(t,k) + c_{2}(t,k),
\end{equation} where $c_{1}(t,k)$ is called \textit{sample optimism} and $c_{2}(t,k)$ is called \textit{bootstrap optimism} for arm $k$ at time $t$.

\textbf{Sample Optimism.} The sample optimism $c_{1}(t,k)$ serves as a control on the event that ''the realized sample estimate discrepancy (ED)  is bounded by sample OED'':

\begin{subequations}
\label{event: sampling_concentration}
\begin{align}
    & E_{t,k} := \{|\Hat{\mu}_{k,t} - \mu_k| \leq c_{1}(t, k)  \norm{\bx_{k}}_{\bV_{t}^{-1}}, \}\label{eq:SamConcOneArm} \\
    & E_{t} := \bigcap_{k=1}^{K} E_{t,k}, 
    \label{eq:SamConcAllArm}
\end{align}
\end{subequations}
where $c_{1}(t, k)$ is a constant which can be tuned by our \texttt{LinReBoot} algorithm, making the bad event $\Bar{E}_{t,k}$ and $\Bar{E}$ become unlikely. In fact, this $E_{t,k}$ is the event that the least squared estimation is "close" to the true mean reward for arm $k$ at round $t$. In section \ref{sec:main_product}, the probability of the bad event $\Bar{E}_t$ is controlled by a parameter tuned by users based on lemma \ref{lemma: sampling_concentration}.

\textbf{Bootstrap Optimism.} 

The bootstrap optimism $c_{2}(t,k)$ serves as a control on the event that ''the realized bootstrap ED is bounded by bootstrap OED'':
\begin{subequations}
\label{event: resampling_concentration}
\begin{align}
    & E_{t,k}^{\prime} := \{|\Tilde{\mu}_{k,t} - \Hat{\mu}_{k,t}| \leq c_{2}(t, k) \norm{\bx_{k}}_{\bV_{t}^{-1}} \},
    \label{eq:ReSamConcOneArm}\\
    & E_{t}^{\prime} := \bigcap_{k=1}^{K} E_{t,k}^{\prime},
    \label{eq:ReSamConcAllArm}
\end{align}
\end{subequations}
where $c_{2}(t, k)$ is also a constant controlling the conditional probability of the bad event $\Bar{E}_{t}^{\prime}$. This $c_{2}(t, k)$ can be tuned by our \texttt{LinReBoot} algorithm as well. Similar to $E_{t,k}$, this $E_{t,k}^{\prime}$ is the event that the residual bootstrap based estimation is "close" to the least squared estimate $\Hat{\mu}_{k,t}$ for arm $k$ at round $t$. In section \ref{sec:main_product}, the probability of bad event $\Bar{E^{\prime}_{t}}$ is controlled by a parameter tuned by users based on lemma \ref{lemma: resampling_concentration}.

\subsection{Optimism Design}

\textbf{Choice of sample optimism ($\alpha$).}
The goal of this part is to illustrate how to pick the sample OED such that the event \eqref{event: sampling_concentration} holds with probability at least $1-\alpha$ for a given confidence budget $\alpha \in (0,1)$. 
Formally, the goal is to find a sample OED function $c_{1}(t, k):[n]\times [K] \mapsto  \mathbb{R}$ such that the event \eqref{eq:SamConcOneArm} holds with probability at least $1-\alpha_{k}$.
To meet the purpose of the risk control, we specify the sample OED function with form 
\begin{equation}\label{eq:sample_OED}
    c_{1}(t, k):= R_2 \sqrt{d \log((1 + tL^{2}/\lambda)/\alpha_{k})} + \lambda^{1/2} S_2 .
\end{equation}
Lemma \ref{lemma: sampling_concentration} gives the formal result on why such choice has confidence budget at most $\alpha_{k}$. For regret analysis,  define $\alpha_{\min}= \underset{k \in [K]}{\min} \alpha_k $ and $\balpha = (\alpha_1,...,\alpha_K)^{\top}$.


\textbf{Choice of bootstrap optimism ($\beta$).}
The goal of this part is to pick bootstrapped OED such that the event \eqref{event: resampling_concentration} holds with probability at least $1-\beta$ for given confidence budget $\beta \in (0,1)$. 
Formally, the goal is to find a sample OED function $c_{2}(t, k):[n]\times [K] \mapsto  \mathbb{R}$ such that the event \eqref{eq:ReSamConcOneArm} holds with probability at least $1-\beta_{k}$.
To meet the purpose of the risk control, we specify the bootstrapped OED function with form 
\begin{equation}\label{eq:bootstrap_OED}
c_{2}(t, k):=
\sqrt{
(2 \sigma_{\omega}^{2} RSS_{k,t} \log(2/\beta_{k}))
/
s_{k,t-1}^2 \norm{\bx_{k}}_{\bV_{t}^{-1}}^2 
} .
\end{equation}

Lemma \ref{lemma: resampling_concentration} gives the formal result on why such choice has a confidence budget at most $\beta_{k}$. For regret analysis, let $\beta_{\min}$ be the smallest $\beta_k \mbox{, } \forall k \in [K]$ and $\bbeta = (\beta_1,...,\beta_K)^{\top}$.

\subsection{Optimism for Optimal Arm}

\textbf{Sample-Bootstrap OED ratio of the optimal arm (b).} Indicated by the regret analysis in \citep{kveton2019perturbed}, instead of controlling the exploration independently, the relation between two sources of explorations needs to be considered because this relation is critical for finding the optimal action. To meet such observation, we define a good event,
\begin{equation}
\label{event: anti_concentration}
\begin{aligned}
    E_{t}^{\prime \prime} := 
    \{\Tilde{\mu}_{1,t} - \Hat{\mu}_{1,t} > c_{1}(t, 1) \norm{\bx_{1}}_{\bV_{t}^{-1}} \}.
\end{aligned}
\end{equation}
Given the good event $E_{t}^{\prime \prime}$, the policy index $\Tilde{\mu}_{1,t}$ of the optimal arm enjoys further positive bias, hence the agent will have better chance to make optimal action.

In particular, we highlight a constant $b$ used to measure the ratio of the sample optimism \eqref{eq:sample_OED} to the bootstrap optimism \eqref{eq:bootstrap_OED}; formally, we require $b$ satisfies
\begin{equation}\label{eq:Samp-Boot_ratio}
    c_{1}(t, 1)/c_{2}(t, 1) \ge b \cdot \sqrt{2 \log \left(2/\beta_{1}\right)}.
\end{equation}
Intuitively, the constant $b$ measures the relation between sample OED and bootstrap OED of the optimal arm. This $b$ plays an important role of the probability lower bound of event \eqref{event: anti_concentration} (See Lemma \ref{lemma: anti_concentration}). Note that, 
if \eqref{eq:Samp-Boot_ratio} holds, we have the lower bound \eqref{eq:ant1}
; otherwise, we have the lower bound \eqref{eq:ant2}. In both cases, we have a lower bound for the event \eqref{event: anti_concentration}.

\textbf{Good event for optimal arm ($\gamma$).} Here we introduce the event that over exploration and under exploration of the optimal arm have been avoided simultaneously. Formally, the constant $\gamma$ is the probability that the bandit index \eqref{eq:policy} is not over-exploration (Event $E_{t}^{\prime}$) and also not under-exploration (Event $E_{t}^{\prime \prime}$)
\begin{equation}\label{eq:not_over_under}
\{c_{1}(t, 1)  <
(\Tilde{\mu}_{1,t} - \Hat{\mu}_{1,t})
/
\norm{\bx_{1}}_{\bV_{t}^{-1}}
< c_{2}(t, 1)  \}.
\end{equation}
Technically, we can show that the probability of the event \eqref{eq:not_over_under} is lower bounded by the term
\begin{equation}
\BP_{t}(E_{t}^{\prime \prime}) - \BP_{t}(\Bar{E}_{t}^{\prime}), 
\end{equation}
with probability at least $1-\gamma$ (Lemma \ref{lemma: connection_concentrations}).
Such lower bound is translated into an upper bound in regret analysis.

\section{Formal Results}\label{sec:main_product}

\begin{table}[t!]
\centering
\renewcommand{\arraystretch}{1.75}

\begin{tabular}{ c|c } 
\hline
Notation & Definition \\
\hline

\multirow{2}{*}{$\zeta_1 (n,d)$} &
$    
(L_2 \sqrt{d \log(\frac{1 + n L^{2}/\lambda}{\alpha_{\min}})} + \lambda^{1/2} S_2)
\times $ \\
 &
$
\sqrt{2(n-K) d \log(1 + \sum_{i=1}^r \sigma_i^2/d \lambda )}
$ \\
\hline

\multirow{2}{*}{$\zeta_2 (n,d)$} &
$
\sqrt{2 \sigma_{\omega}^{2} log(\frac{2}{\beta_{\min}})} \times
$  \\
 & 
$
\sqrt{2(n-K) d \log(1 + \sum_{i=1}^r \sigma_i^2/d \lambda )}
 $ \\ 
\hline

$\zeta_3(n)$ 
&  $2 K \sqrt{4 L_2\sigma_{\omega}^{2} \log(\frac{2}{\beta_{\min}})}
(\log n + 1)$ \\
\hline

$\zeta_4(n)$ 
&
$
2 S_2 L   ((n - K) (\alpha + \beta) + K - 1)
$\\
\hline
\end{tabular}

\caption{\footnotesize Notations in Regret Analysis}
\label{table: regret_bound_n}

\end{table}

\subsection{Regret Bound for \texttt{LinReBoot}}
\label{subsec: regret_bound}

\begin{theorem}
\label{theorem: main}
Under Assumptions \ref{ass:bound}, \ref{ass:noise_bound}, \ref{ass: lower_bound} and technical conditions \eqref{appendix: proof_theorem_5.1_rho} and \eqref{requirement_for_b},  with probability at least $1-(\delta + \gamma)$, the expected regret of Algorithm \ref{alg:LinReBoot: Version_1}  is bounded as,
\begin{equation}
\begin{aligned}
    R_{n}
    \leq
    & C_1 (\alpha_1, \bbeta, \gamma, b) \zeta_1 (n,d) \\
    + 
    & C_2 (\balpha, \bbeta, \gamma, b, \delta) \zeta_2 (n,d) \\
    +
    & C_1 (\alpha_1, \bbeta, \gamma, b) \zeta_3 (n) +
    \zeta_4 (n),
\end{aligned}
\end{equation}
where $\zeta_1$, $\zeta_2$, $\zeta_3$ and $\zeta_4$ are defined in Table.\ref{table: regret_bound_n} and $C_1 $, $C_2$, $M_1$, $M_2$ are described in Table.\ref{table: regret_bound_constants}.
\end{theorem}
\begin{proof}

See Appendix \ref{appendix: proof_theorem_main}.

\end{proof}

\begin{corollary}
\label{corollary: rate}
Let $\balpha = \bbeta = \frac{1}{\sqrt{n}}\bOnes$, the order of high probability upper bound in Theorem \ref{theorem: main} is $\Tilde{O}(d \sqrt{n})$.
\end{corollary}
\begin{proof}

See Appendix \ref{appendix: proof_corollary_rate}.

\end{proof}
Corollary \ref{corollary: rate} shows that our regret bound scales as the regret bound of Linear Thompson sampling \citep{agrawal2013thompson} and Linear PHE \citep{kveton2019perturbed}.

\begin{figure*}[th!]
\centering

\includegraphics[scale = 0.2]{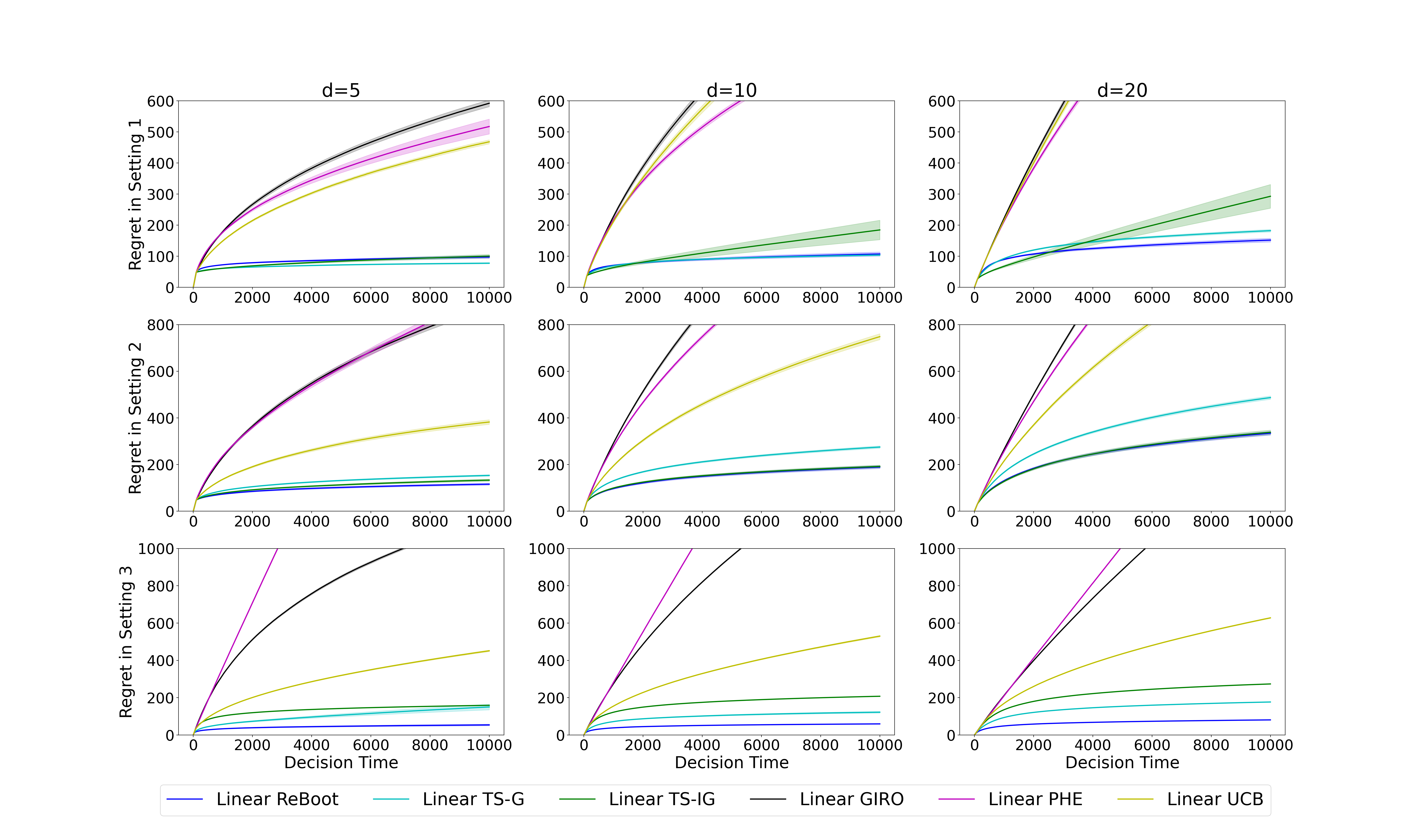}

\caption{Comparison of \texttt{LinReBoot} with Gaussian Bootstrap weights to baselines under three linear bandit problems and three different context dimension $d$. First row referred to the setting in Section \ref{subsec: SLB_experiment}, second row is for Section \ref{subsec: LB_random_experiment} and the last row is for Section \ref{subsec: LB_covariates_experiment}. Three columns refer to $d=5$, $d=10$ and $d=20$ respectively.}
\label{fig: summary}

\end{figure*}

\subsection{Validate Sample Optimism}
\label{subsec: sampling_concentration}

\begin{lemma}
\label{lemma: sampling_concentration}
Under Assumptions \ref{ass:bound}, \ref{ass:noise_bound}, \ref{ass: lower_bound} and choose $c_{1}(t,k)$ as \eqref{eq:sample_OED},  $\BP(\Bar{E}_{t,k})$, the probability of bad event corresponded to least squared estimation described in (\ref{event: sampling_concentration}), is controlled. Formally, $\forall k \in [K]$, $\forall \alpha_{k} > 0$, $\forall t \geq 1$,
\begin{equation}
    \BP( |\Hat{\mu}_{k,t} - \mu_{k}| 
    \leq
    c_{1}(t, k)\norm{\bx_{k}}_{\bV_{t}^{-1}} )
    \geq
    1 - \alpha_{k}.
\end{equation}
Consequently, we have $\BP(\Bar{E}_t) \leq \alpha := \sum_{k=1}^{K} \alpha_{k}$.
\end{lemma}
\begin{proof}

See Appendix \ref{appendix: proof_lemma_sampling_concentration}.

\end{proof}
Lemma \ref{lemma: sampling_concentration} supports that the choice  of $c_{1}(t,k)$ at \eqref{eq:sample_OED} for the sample optimism event \eqref{event: sampling_concentration} is valid with confidence budget $\alpha$.

\subsection{Validate Bootstrap Optimism}
\label{subsec: resampling_concentration}

\begin{lemma}
\label{lemma: resampling_concentration}
Suppose bootstrap weights are Gaussian. Pick $c_{2}(t,k)$ as \eqref{eq:bootstrap_OED}. The conditional probability of bad event corresponding to residual bootstrap exploration described in \eqref{event: resampling_concentration}, $\BP_t(\Bar{E}_{t,k}^{\prime})$, is controlled. Formally,
$\forall k \in [K]$, $\forall \beta_{k} > 0$, $\forall t \geq 1$
\begin{equation}
    \BP_{t} (|\Tilde{\mu}_{k,t} - \Hat{\mu}_{k,t}| 
    \leq
    c_{2}(t, k)\norm{\bx_{k}}_{\bV_{t}^{-1}})
    \geq
    1 - \beta_{k} .
\end{equation}
Consequently, we have $\BP_{t}(\Bar{E^{\prime}_{t}}) \leq \beta := \sum_{k=1}^{K} \beta_{k}$.
\end{lemma}

\begin{proof}
See Appendix \ref{appendix: proof_lemma_resampling_concentration}.
\end{proof}

Lemma \ref{lemma: resampling_concentration} supports that the choice  of $c_{2}(t,k)$ at \eqref{eq:bootstrap_OED} for the sample optimism event \eqref{event: resampling_concentration} is valid with confidence budget $\beta$.

\subsection{Sample-Bootstrap ratio}
\label{subsec: anti_concentration}

\begin{lemma}
\label{lemma: anti_concentration}
Under Assumptions \ref{ass:bound}, \ref{ass:noise_bound}, \ref{ass: lower_bound}. Suppose bootstrap weights are Gaussian. The conditional probability of anti-concentration for optimal arm described in (\ref{event: anti_concentration}), $\BP_t(\Bar{E}_{t}^{\prime \prime})$, has lower bound. Formally, if $b$ satisfies \eqref{eq:Samp-Boot_ratio},
\begin{equation}
\begin{aligned}
    \BP_{t}(E^{\prime \prime}_{t}) 
    \geq 
    \frac{b}{\sqrt{2 \pi}} 
    \exp(-\frac{3 c_{1}^2(t,1) s_{1, t-1}^2 \norm{\bx_{1}}_{\bV_{t}^{-1}}^2}{ 2 \sigma_{\omega}^2 RSS_{1,t}}) .
    \label{eq:ant1}
\end{aligned}
\end{equation}
Otherwise, 
\begin{equation}
\begin{aligned}
    \BP_{t}(E^{\prime \prime}_{t}) 
    \geq 
    \Phi(-b)
    \label{eq:ant2}, 
\end{aligned}
\end{equation}
where $\Phi$ is the CDF of standard normal distribution.
\end{lemma}
\begin{proof}

See Appendix \ref{appendix: proof_lemma_anti_concentration}.

\end{proof}

Lemma \ref{lemma: anti_concentration} provides the lower bound result for good event $E_{t}^{\prime \prime}$. The result indicates that, if the bootstrap optimism is not 'too large', then the \texttt{LinReBoot} procedure can enjoy additional regret reduction.

\subsection{Validate good event}

\label{subsec: connection_concentrations}
\begin{lemma}
\label{lemma: connection_concentrations}
Under Assumptions \ref{ass:bound}, \ref{ass:noise_bound}, \ref{ass: lower_bound} and suppose Bootstrap weights are Gaussian. Assume $b$ satisfies a technical condition \eqref{requirement_for_b}. Then, with probability at least $1- \gamma$, $\BP_{t}(E_{t}^{\prime \prime}) - \BP_{t}(\Bar{E}_{t}^{\prime})$ has lower bound, 
\begin{equation}
\begin{aligned}
\dfrac{b}{\sqrt{2 \pi}} \exp( - \frac{3 s_{1, t-1}^{3/2} c_{1}^{2}(t,1) \norm{\bx_{1}}^{2}_2 }{8 \sigma^{2}_{\omega} (\sigma_{\min}^2 + \lambda) \sqrt{\frac{1}{M_2} \log(\frac{M_1}{1 - \gamma})}}) -\beta,
\end{aligned}
\end{equation}
where $M_1$ and $M_2$ are defined in Table.\ref{table: regret_bound_constants}.
\end{lemma}
\begin{proof}

See Appendix \ref{appendix: proof_lemma_connection_concentrations}.

\end{proof}

Lemma \ref{lemma: connection_concentrations} provided the a high probability lower bound for the difference between probability of the event for anti-concentration $E_{t}^{\prime \prime}$ and probability of bad event discussed in bootstrap optimism in Section \ref{sec:colla_opti}. This lower bound is also for probability of `not under and not over exploration' event \eqref{eq:not_over_under}. Lemma \ref{lemma: connection_concentrations} links the sample optimism and bootstrap optimism and holds a right amount of exploration of the optimal arm.

\section{Experiments}\label{sec:exp_to_imple}

In this section, we conduct empirical studies under three settings: Stochastic Linear Bandit, Contextual Linear Bandit and Linear Bandit with Covariates. Our \texttt{LinReBoot} is compared to several baselines including \texttt{LinTS-G} \citep{agrawal2013thompson, lattimore2020bandit}, \texttt{LinTS-IG} \citep{honda2014optimality, riquelme2018deep}, \texttt{LinPHE} \citep{kveton2019perturbed}, \texttt{LinGIRO} \citep{kveton2019garbage} and \texttt{LinUCB}  \citep{abbasi2011improved, lattimore2020bandit} . More details about baselines can be found in Appendix \ref{appendix: experiment_other_algs}.

\subsection{Stochastic Linear Bandit}
\label{subsec: SLB_experiment}

We compare \texttt{LinReBoot} to other linear bandit algorithms under stochastic linear bandit described in Section \ref{sec: Stochastic_Linear_Bandit}. 
We experiment with several dimensions $d$ including $5$, $10$ and $20$. $K$ is chosen as $100$. Synthetic data generation for this setting is deferred to Appendix \ref{appendix: experiment_settings} in the supplementary material.
\textbf{Results.} The first row of Figure \ref{fig: summary} reports the results for Stochastic Linear Bandit setting. 
Our \texttt{LinReBoot} rivals \texttt{LinTS-G} and \texttt{LinTS-IG} while substantially exceeds \texttt{LinGIRO}, \texttt{LinPHE} and \texttt{LinUCB}. When $d$ increases, the performance of \texttt{LinReBoot} rivals and exceeds the best of other methods.

\subsection{Contextual Linear Bandit}
\label{subsec: LB_random_experiment}

In the second experiment, we compare \texttt{LinReBoot} to other linear bandit algorithms under Contextual Linear Bandit where the contexts are generated from some distributions by arms. Note that this setting matches previous work \citep{chu2011contextual}. Linear bandit algorithms can also be applied under this kind of environment. In our experiment, the \texttt{LinReBoot} is implemented as Algorithm \ref{alg:LinReBoot: Version_2} in Appendix \ref{appendix: experiment_algs}. Like the setting in Section \ref{subsec: SLB_experiment}, the dimension of $d$ is chosen as $5$ or $10$ or $20$ and the synthetic data generation for this setting is described in Appendix \ref{appendix: experiment_settings}. 
\textbf{Results.} The second row of Figure \ref{fig: summary} reports the results for Contextual Linear Bandit. Our \texttt{LinReBoot} rival \texttt{LinTS-G} and substantially exceed \texttt{LinTS-IG}, \texttt{LinGIRO}, \texttt{LinPHE} and \texttt{LinUCB}. When $d$ increases, the performance of \texttt{LinReBoot} rivals \texttt{LinTS-IG} and exceeds others.

\subsection{Bandit with Covariates}
\label{subsec: LB_covariates_experiment}

Our last experiment is conducted under the setting of linear bandit with covariates, which is also called linear parametrized bandit by \citep{rusmevichientong2010linearly}. This problem is significantly different from the previous two problems in the following ways. Each arm has its true parameter $\btheta_k$.  That is, each arm has its estimate $\Hat{\btheta}_k$ from the ridge regression procedure in Section \ref{subsec: LinReBoot}. Also, unlike the setting in Section \ref{subsec: LB_random_experiment}, the contexts are generated from a distribution that is independent of arms. Thus the overall task in this setting is not only the estimation of the target parameter $\btheta$, but also the detection of which arm a context belongs to. This case is also referred to as the online decision-making under covariates \citep{bastani2020online}. For the \texttt{LinReBoot} in this setting, detailed algorithm is provided as Algorithm \ref{alg:LinReBoot: Version_3} in Appendix \ref{appendix: experiment_algs}. $d$ is chosen as $5$ or $10$ or $20$ and $K=10$. Synthetic data generation for this setting is described in Appendix \ref{appendix: experiment_settings}. 
\textbf{Results.} The third row of Figure \ref{fig: summary} reports the results for Linear Bandit with Covariates. Our \texttt{LinReBoot} exceeds all competing algorithms \texttt{LinTS-G}, \texttt{LinTS-IG}, \texttt{LinGIRO}, \texttt{LinPHE} and \texttt{LinUCB}.

\textbf{Summary.} 
From Figure \ref{fig: summary}, the proposed \texttt{LinReBoot} is always the top 3 algorithms under all settings and all choice of dimension $d$. More specifically, \texttt{LinReBoot} is clearly comparable to the state-of-the-art Linear Thompson Sampling algorithms(\texttt{LinTS-G}, \texttt{LinTS-IG}) or even outperforms them in many cases. Regarding the computational cost, from Table.\ref{table: comp_cost}, our proposed \texttt{LinReBoot} is consistently computational efficient among all settings compared to \texttt{LinTS-G}, \texttt{LinTS-IG} and \texttt{LinUCB} under all three settings.

\section{Conclusion}\label{sec:DisConclu}

We propose \texttt{LinReBoot} algorithm for stochastic linear bandit problems. In theory, we prove \texttt{LinReBoot} that secures $\tilde{O}(d \sqrt{n})$ high probability expected regret. Empirically, we show \texttt{LinReBoot} rivals \texttt{LinTS-G}, \texttt{LinTS-IG} and exceeds  \texttt{LinPHE}, \texttt{LinGIRO} and \texttt{LinUCB}, which supports the easy-generalizability of \texttt{ReBoot} principle in \citep{wang2020residual} under various contextual bandit settings including Stochastic Linear Bandit, Contextual Linear Bandit, and Linear Bandit with Covariates. 

\clearpage

\vskip 0.2in
\nocite{*}
\bibliography{wu_32}

\clearpage

\onecolumn

\title{Residual Bootstrap Exploration for Stochastic Linear Bandit (Supplementary Materials)}

\appendix

\maketitle

\section{Proofs of Main Results}\label{appendix: main}

\subsection{Proof of Theorem \ref{theorem: main}}
\label{appendix: proof_theorem_main}
\begin{proof}
The regret bound analysis of algorithm \ref{alg:LinReBoot: Version_1} involves several key Lemmas and conditions. Inspired by the definition of expected regret, one key Lemma is providing the upper bound for expected optimal gap given the history $\mF_{t-1}$ at round $t$, $\BE_{t}[\Delta_{I_{t}}]$. This is similar to the proof in other linear bandit algorithms such as \texttt{LinPHE} \citep{kveton2019perturbed} and \texttt{LinUCB} \citep{abbasi2011improved}. Lemma \ref{appendix: lemma_gap} in the following part gives this result. The other important Lemma is bounding sum of expected `square root of normalized RSS' which is described in Lemma \ref{appendix: lemma_lm_RSS}. The Third key result, Lemma \ref{appendix: lemma_norm}, is an algebra result from \citep{abbasi2011improved} which bounds the sum of action context norms. Moreover, Lemmas in Section \ref{sec:main_product} play essential roles in regret bound analysis. Lemma \ref{lemma: sampling_concentration} and Lemma \ref{lemma: resampling_concentration} control the sample optimism and bootstrap optimism respectively. Lemma \ref{lemma: anti_concentration} gives lower bound for the event of anti-concentration, which is necessary lower bound for analyzing exploration in linear bandit algorithms. Another key step is carefully evaluating anti-concentration and its connection to concentration, which is summarised by lemma \ref{lemma: connection_concentrations}. An technical condition about tuning parameter $\sigma_{\omega}^2$, which will be discussed later in this proof is also needed for regret analysis. We start from listing the Lemmas and condition and main proof of Theorem \ref{theorem: main} will be given later.

\begin{lemma}
\label{appendix: lemma_gap}
Assume the same as Theorem \ref{theorem: main}. Suppose $M \geq \underset{k \in [K]}{\max} \mbox{ } \Delta_k$. When $c_{1}(t,k), c_{2}(t,k) \geq 1$ and $\BP_{t}(E_{t}^{\prime \prime}) - \BP_{t}(\Bar{E}_{t}^{\prime}) > 0$ for $\forall t > K$ and $\forall k \in [K]$, then on event $E_{t}$, almost surely,
\begin{equation}
\begin{aligned}
    \BE_{t}[\Delta_{I_{t}}] \leq (\dfrac{2}{\BP_{t}(E_{t}^{\prime \prime}) - \BP_{t}(\Bar{E}_{t}^{\prime})}+1)
    (c_{1}(t,I_t)+c_{2}(t,I_t))\BE_{t}[\norm{\bx_{I_t}}_{\bV_{t}^{-1}}] + M\BP(\Bar{E}^{\prime}_{t})
\end{aligned}
\end{equation}
\end{lemma}
\begin{proof}
See appendix \ref{appendix: proof_lemma_gap}
\end{proof}
\begin{remark}
Lemma \ref{appendix: lemma_gap} provides the upper bound for expected optimal gap given the latest history. This result directly impacts the upper bound of expected regret of \texttt{LinReBoot}, which means that each terms in the upper bound given by Lemma \ref{appendix: lemma_gap} need to be further bounded. As we expect, sample optimism $(c_{1}(t,I_t) \BE_{t}[\norm{\bx_{I_t}}_{\bV_{t}^{-1}})$ and Bootstrap optimism $(c_{2}(t,I_t) \BE_{t}[\norm{\bx_{I_t}}_{\bV_{t}^{-1}})$ require further bounding. An interesting observation is the appearance of term $\BP_{t}(E_{t}^{\prime \prime}) - \BP_{t}(\Bar{E}_{t}^{\prime})$ which is the lower bound of probability of $E_t^{\prime \prime}$ defined in \eqref{eq:not_over_under}. Intuitively, this event connects the exploration from ridge estimation and the exploration from residual Bootstrapping and iF the lower bound $\BP_{t}(E_{t}^{\prime \prime}) - \BP_{t}(\Bar{E}_{t}^{\prime})$ is too small, then this upper bound in Lemma \ref{appendix: lemma_gap} becomes trivial, which means our regret analysis become meaningless.
\end{remark}

\begin{lemma}
\label{appendix: lemma_lm_RSS}
Assume the same as Theorem \ref{theorem: main}. With probability at least $1-\delta$,
\begin{equation}
\begin{aligned}
    \sum_{t=K+1}^{n} 
    \BE[\sqrt{\dfrac{ RSS_{I_t,t}} {s_{I_t,t-1}^2}}] 
    \leq
    \sqrt{2} (L_2 \sqrt{
    r \log(1 + \sigma_{\max}^2/\lambda) + 2log(\frac{1}{\delta})
    } + \lambda^{1/2} S_2)
    \sum_{t = K+1}^n \BE[ \norm{\bx_{I_t}}_{\bV_t^{-1}}] 
    +  2 \sqrt{2} K\sqrt{L_2}(\log n + 1)
\end{aligned}
\end{equation}
\end{lemma}
\begin{proof}
See appendix \ref{appendix: proof_lemma_lm_RSS}
\end{proof}
\begin{remark}
Lemma \ref{appendix: lemma_lm_RSS} is bounding sum of expected `square root of normalized RSS', that is, $\sqrt{ RSS_{I_t,t} /s_{I_t,t-1}^2}$. As discussed in Section \ref{sec:TheoConsi}, the RSS contributes additional exploration. As a matter of fact, the `square root of normalized RSS' is proportional to the variance of Bootstrapped mean. Consequently, this Lemma assists bounding of the magnitude of extra exploration from residual Bootstrapping.
\end{remark}

\begin{lemma}
\label{appendix: lemma_norm}
Assume the same as Theorem \ref{theorem: main}. Then
\begin{equation}
\begin{aligned}
    \sum_{t=K+1}^{n} \norm{\bx_{I_t}}_{\bV_{t}^{-1}} \leq 
    \sqrt{2(n-K) d \log(1 + \frac{ \sum_{i=1}^r \sigma_i^2}{ d \lambda})}
\end{aligned}
\end{equation}
\end{lemma}
\begin{proof}
See appendix \ref{appendix: proof_lemma_norm}
\end{proof}
\begin{remark}
Lemma \ref{appendix: lemma_norm} bounds the sum of action context norms which is also bounded in regret analysis of most contextual bandit algorithms.
\end{remark}

\textbf{Technical Condition.} Suppose for any $K < t \leq n$ and some $\rho > 0$ such that $\rho = \Tilde{O}(1)$ with respect to $n$ and $d$. Then
\begin{equation}
\label{appendix: proof_theorem_5.1_rho}
\begin{aligned}
s_{1, t-1}^{3/2} c_{1}^{2}(t,1) \leq \rho 
\sigma^{2}_{\omega} (\sigma_{\min}^2 + \lambda) \sqrt{\frac{1}{M_2} \log(\frac{M_1}{1 - \gamma})}
\end{aligned}
\end{equation}

\begin{remark}
This condition indicates that there is a lower bound for $\sigma^{2}_{\omega}$, which means the extra exploration contributes to bounding of expected regret. This lower bound strongly supports the necessity of residual Bootstrap exploration. Another observation is that the lower bound is related to the time $t$ and the number of pulling of optimal arm, which means that this hyperparameter for exploration $\sigma^{2}_{\omega}$ should depend on decision round $t$. However, since $\sigma^{2}_{\omega}$ is also related to some fixed constant related to environment and $\rho$ which is a order of logarithm terms of $n$ and $t$, it remains hard to determine what is the exact relation between  $\sigma^{2}_{\omega}$ and $n$. This lower bound is only providing the conservative guarantee that the regret bound is sub-linear.
\end{remark}

\textbf{Main proof of Theorem \ref{theorem: main}.}\\
Following part is the main proof of Theorem \ref{theorem: main}, starting from decomposing regret by events,
\begin{subequations}
\begin{align}
    R_{n} 
    &= \sum_{k=2}^{K} \Delta_{k} \BE[\sum_{t=1}^{n} \BI\{I_{t} = k\}]\\
    &= \sum_{t=1}^{n} \BE[\Delta_{I_{t}}]\\
    &= \sum_{t=K+1}^{n} \BE[\Delta_{I_{t}}] + \sum_{t=1}^{K} \BE[\Delta_{I_{t}}] \\
    & \leq  \sum_{t=K+1}^{n} \BE[\Delta_{I_{t}} \BI\{E_{t}\}] + \sum_{t=K+1}^{n} \BE[\Delta_{I_{t}} \BI\{\Bar{E}_{t}\}] + 2 S_2 L   (K-1)
    \quad (\text{by (\ref{appendix: proof_theorem_5.1_a})})\\
    & \leq \sum_{t=K+1}^{n} \BE[\Delta_{I_{t}} \BI\{E_{t}\}] + 2 S_2 L   (n - K)\BP(\Bar{E}_t) + 2 S_2 L   (K-1)
    \quad (\text{by (\ref{appendix: proof_theorem_5.1_a})})\\
    &= \sum_{t=K+1}^{n} \BE[\BE_{t}[\Delta_{I_{t}} \BI\{E_{t}\}]] + 2 S_2 L   (n - K)\BP(\Bar{E}_t) + 2 S_2 L  (K-1)\\
    & \leq \sum_{t=K+1}^{n} \BE[(\dfrac{2}{\BP_{t}(E_{t}^{\prime \prime}) - \BP_{t}(\Bar{E}_{t}^{\prime})}+1)
    (c_{1}(t,I_t)+c_{2}(t,I_t))\BE_{t}[\norm{\bx_{I_t}}_{\bV_{t}^{-1}}]] \notag \\
    & \quad + 2 S_2 L  (\sum_{t=K+1}^{n} \BE[\BP(\Bar{E}^{\prime}_{t})] + (n - K)\BP(\Bar{E}_t) + K - 1) 
    \quad (\text{by lemma \ref{appendix: lemma_gap}})\\
    & \leq \sum_{t=K+1}^{n} \BE[(\dfrac{2}{\BP_{t}(E_{t}^{\prime \prime}) - \BP_{t}(\Bar{E}_{t}^{\prime})}+1)
    (c_{1}(t,I_t)+c_{2}(t,I_t))\BE_{t}[\norm{\bx_{I_t}}_{\bV_{t}^{-1}}]] \notag \\
    & \quad + 2 S_2 L   ((n - K) (\alpha + \beta) + K - 1) 
    \quad (\text{by lemma \ref{lemma: sampling_concentration} and \ref{lemma: resampling_concentration}})
\end{align}
\end{subequations}
Where (\ref{appendix: proof_theorem_5.1_a}) is upper bound of optimal gap, that is, $\forall k \in [K]$
\begin{equation}
\label{appendix: proof_theorem_5.1_a}
\begin{aligned}
    \Delta_k 
    & = \btheta^{\top}(\bx_{1} - \bx_{k})\\
    & \leq \norm{\btheta}_2 \norm{\bx_{1} - \bx_{k}}_2 \\
    & \leq \norm{\btheta}_2 \sqrt{2 \norm{\bx_{1}}_2^2 + 2 \norm{\bx_{k}}_2^2} \\
    & \leq 2 S_2 L 
\end{aligned}
\end{equation}

By lemma \ref{lemma: connection_concentrations} and the technical condition \eqref{appendix: proof_theorem_5.1_rho},
\begin{subequations}
\begin{align}
    \dfrac{2}{\BP_{t}(E_{t}^{\prime \prime}) - \BP_{t}(\Bar{E}_{t}^{\prime})} 
    \leq & 
    \dfrac{2}{
    \dfrac{b}{\sqrt{2 \pi}} \exp( - \frac{3 s_{1, t-1}^{3/2} c_{1}^{2}(t,1) \norm{\bx_{1}}^{2}_2 }{8 \sigma^{2}_{\omega} (\sigma_{\min}^2 + \lambda) \sqrt{\frac{1}{M_2} \log(\frac{M_1}{1 - \gamma})}}) -\beta} \\
    \leq &
    \dfrac{2}{
    \dfrac{b}{\sqrt{2 \pi}} \exp( - \frac{3}{8}\norm{\bx_{1}}^{2}_2 \rho) -\beta}
\end{align}
\end{subequations}
Where
\begin{align}
    M_1 &:=
    (e-1)^2 \exp(\dfrac{8\sigma_{\max}^2 S_2^2 L_2}{\frac{\lambda^2}{(\sigma_{\max}^2 + \lambda)^2} S_1^2 L_1} - 6) \\
    M_2 &:=
    \dfrac{4\sigma_{\max}^2 S_2^2 L_2 - 2 \frac{\lambda^2}{(\sigma_{\max}^2 + \lambda)^2} S_1^2 L_1}{(\frac{\lambda^2}{(\sigma_{\max}^2 + \lambda)^2} S_1^2 L_1)^2} 
\end{align}
Define the following notations for simplicity, note that the following constants are independent of $n$ and $d$, 
\begin{subequations}
\begin{align}
    C_1 (\alpha_1, \bbeta, \gamma, b)
    &:=     
    \dfrac{2}{
    \dfrac{b}{\sqrt{2 \pi}} \exp( - \frac{3}{8}\norm{\bx_{1}}^{2}_2 \rho) -\beta}+1  \\
    C_2 (\balpha, \bbeta, \gamma, b, \delta)
    &:=
    C_1 (\alpha_1, \bbeta, \gamma, b) \times
    \sqrt{2} (L_2 \sqrt{
    r \log(1 + \sigma_{\max}^2/\lambda) + 2\log(\frac{1}{\delta})
    } + \lambda^{1/2} S_2)
\end{align}
\end{subequations}
Then, with probability at least $1-\gamma$,
\begin{subequations}
\begin{align}
    R_{n}
    & \leq  
    C_1 (\alpha_1, \bbeta, \gamma, b)  \sum_{t=K+1}^{n}
    \BE[(c_{1}(t,I_t)+c_{2}(t,I_t))\BE_{t}[\norm{\bx_{I_t}}_{\bV_{t}^{-1}}]] 
    \notag \\
    & \quad + 2 S_2 L   ((n - K) (\alpha + \beta) + K - 1) \\
    & =
    C_1 (\alpha_1, \bbeta, \gamma, b)  \sum_{t=K+1}^{n}
    \BE[c_{1}(t,I_t)\BE_{t}[\norm{\bx_{I_t}}_{\bV_{t}^{-1}}]] 
    \notag \\
    & \quad + 
    C_1 (\alpha_1, \bbeta, \gamma, b)  \sum_{t=K+1}^{n}
    \BE[c_{2}(t,I_t)\BE_{t}[\norm{\bx_{I_t}}_{\bV_{t}^{-1}}]]
    \notag \\
    & \quad + 2 S_2 L   ((n - K) (\alpha + \beta) + K - 1) \\
    & \leq
    C_1 (\alpha_1, \bbeta, \gamma, b)  
    (L_2 \sqrt{d \log(\frac{1 + n L^{2}/\lambda}{\alpha_{\min}})} + \lambda^{1/2} S_2)
    \sum_{t=K+1}^{n}
    \BE[\norm{\bx_{I_t}}_{\bV_{t}^{-1}}]
    \notag \\
    & \quad + 
    C_1 (\alpha_1, \bbeta, \gamma, b)
    \sum_{t=K+1}^{n}
    \BE[\sqrt{\dfrac{2 \sigma_{\omega}^{2} RSS_{I_t,t} \log(\frac{2}{\beta_{I_t}})}
    {s_{I_t,t-1}^2}}]
    \notag \\
    & \quad + 2 S_2 L   ((n - K) (\alpha + \beta) + K - 1) \\
    & \leq
    C_1 (\alpha_1, \bbeta, \gamma, b)
    (L_2 \sqrt{d \log(\frac{1 + n L^{2}/\lambda}{\alpha_{\min}})} + \lambda^{1/2} S_2)
    \sum_{t=K+1}^{n}
    \BE[\norm{\bx_{I_t}}_{\bV_{t}^{-1}}]
    \notag \\
    & \quad + 
    C_1 (\alpha_1, \bbeta, \gamma, b)
    \sqrt{2 \sigma_{\omega}^{2} \log(\frac{2}{\beta_{\min}})}
    \sum_{t=K+1}^{n} 
    \BE[\sqrt{\dfrac{ RSS_{I_t,t}} {s_{I_t,t-1}^2}}]
    \notag \\
    & \quad + 2 S_2 L   ((n - K) (\alpha + \beta) + K - 1) 
\end{align}
\end{subequations}

Further define,
\begin{align}
    \zeta_1 (n,d)
    &:= 
    (L_2 \sqrt{d \log(\frac{1 + n L^{2}/\lambda}{\alpha_{\min}})} + \lambda^{1/2} S_2)
    \sqrt{2(n-K) d \log(1 +  \sum_{i=1}^r \sigma_i^2/d \lambda )}
    \\
    \zeta_2 (n,d)
    &:= 
    \sqrt{2 \sigma_{\omega}^{2} \log(\frac{2}{\beta_{\min}})}
    \sqrt{2(n-K) d \log(1 +  \sum_{i=1}^r \sigma_i^2/d \lambda )}\\
    \zeta_3 (n)
    &:= 
    2 K \sqrt{4 L_2\sigma_{\omega}^{2} log(\frac{2}{\beta_{\min}})}
    (\log n + 1)
     \\
    \zeta_4 (n)
    &:= 
    2 S_2 L   ((n - K) (\alpha + \beta) + K - 1) 
\end{align}
By lemma \ref{appendix: lemma_lm_RSS}, with probability at least $1-(\delta + \gamma)$,
\begin{equation}
\begin{aligned}
    R_{n}
    & \leq
    C_1 (\alpha_1, \bbeta, \gamma, b) \zeta_1 (n,d) + 
    C_2 (\balpha, \bbeta, \gamma, b, \delta) \zeta_2 (n,d) +
    C_1 (\alpha_1, \bbeta, \gamma, b) \zeta_3 (n,d) +
    \zeta_4 (n,d)
\end{aligned}
\end{equation}

The $\zeta_1$, $\zeta_2$, $\zeta_3$ and $\zeta_4$ can also be found in Table.\ref{table: regret_bound_n} and $C_1 $ and $C_2 $  are summarised in the Table.\ref{table: regret_bound_constants}.

\begin{table*}[h]
\centering
\begin{tabular}{ c|c } 
\hline
Notation & Definition \\
\hline
$M_1$ &
$(e-1)^2 \exp(\dfrac{8\sigma_{\max}^2 S_2^2 L_2}{\lambda^2 S_1^2 L_1/ (\sigma_{\max}^2 + \lambda)^2 } - 6)$ \\ 
$M_2$ &
$\dfrac{4\sigma_{\max}^2 S_2^2 L_2 -2 \lambda^2 S_1^2 L_1/ (\sigma_{\max}^2 + \lambda)^2}{(\lambda^2 S_1^2 L_1/ (\sigma_{\max}^2 + \lambda)^2)^2}$ \\ 
$C_1$ &
$
2 \bigg( 
\dfrac{b}{\sqrt{2 \pi}} \exp( - \frac{3}{8}\norm{\bx_{1}}^{2}_2 \rho ) -\beta
\bigg)^{-1}+1
$
\\
$C_2$ &
$
C_1 
\sqrt{2} (L_2 \sqrt{
r \log(1 + \sigma_{\max}^2/\lambda) + 2 \log(\frac{1}{\delta})
} + \lambda^{1/2} S_2)
$ \\
\hline
\end{tabular}
\caption{\footnotesize Constants in Analysis}
\label{table: regret_bound_constants}
\end{table*}

\end{proof}

\subsection{Proof of Corollary \ref{corollary: rate}}
\label{appendix: proof_corollary_rate}
\begin{proof}
We will analyze terms $C_1$, $C_2$ and $\zeta_1$, $\zeta_2$, $\zeta_3$,  $\zeta_4$ one by one in terms of the rate in the big $O$ notation with respect to $n$ and $d$. Also recall that the notation $\Tilde{O}$ is the big $O$ notation up to logarithmic factor with respect to $n$ and $d$. Following steps include the first step for $C_1$ and $C_2$, the second step for $\zeta_1$, $\zeta_2$, $\zeta_3$ and $\zeta_4$ and the last one for combining results.\\
\textbf{Step 1} As $\bbeta$ is chosen as a vector with elements $\frac{1}{\sqrt{n}}$, the term $C_1$ is actually $O(\rho)$ which is assumed to be $\Tilde{O}(1)$. Under stochastic linear bandit that contexts and subgaussian constant $L_2$ are given, $C_2$ is also $\Tilde{O}(1)$. Note that, other parameters such as $\delta$, $\lambda$ and $b$ are viewed as constants.\\
\textbf{Step 2.} From Table.\ref{table: regret_bound_n}, as $\balpha$ is chosen as a vector with elements $\frac{1}{\sqrt{n}}$, we can conclude that $\zeta_1(n,d) = O(\sqrt{d \log n} \times \sqrt{nd \log d})$, $\zeta_2(n,d) = O(\sqrt{\log n} \times \sqrt{nd \log d})$, $\zeta_3(n) = O(\log n \sqrt{\log n})$ and $\zeta_4(n) = O(\sqrt{n})$. By the notation of $\Tilde{O}$, it can be summarised as $\zeta_1(n,d) = \Tilde{O}(d \sqrt{n})$, $\zeta_2(n,d) = \Tilde{O}(\sqrt{dn})$, $\zeta_3(n) = \Tilde{O}(1)$ and $\zeta_4(n) = \Tilde{O}(\sqrt{n})$.\\
\textbf{Step 3.} As a result, expected regret of our \texttt{LinReBoot} in Theorem \ref{theorem: main} under the choice of tuning parameter mentioned in Corollary \ref{corollary: rate}, has high probability  upper bound with the order 
$\Tilde{O}(d \sqrt{n}) + \Tilde{O}(\sqrt{dn}) + \Tilde{O}(1) + \Tilde{O}(\sqrt{n}) = \Tilde{O}(d \sqrt{n})$. 
\end{proof}

\subsection{Proof of Lemma \ref{lemma: sampling_concentration}}
\label{appendix: proof_lemma_sampling_concentration}
\begin{proof}
Based on Theorem 2 in \citep{abbasi2011improved} which is Lemma \ref{appendix: Ellipsoid}, for all $\alpha \in (0,1)$,
\begin{equation}
\begin{aligned}
    \BP(\norm{\btheta - \Hat{\btheta}_t}_{\bV_t} 
    \leq 
    L_2 \sqrt{d \log(\frac{1 + tL^{2}/ \lambda}{\alpha})} + \lambda^{1/2} S_2)
    \geq 1-\alpha
\end{aligned}
\end{equation}
Thus, $\forall \alpha_k \in (0,1)$, with probability at least $1-\alpha_k$
\begin{subequations}
\begin{align}
    |\Hat{\mu_{k,t}} - \mu_{k}| 
    &= |\bx^{\top} (\Hat{\btheta}_t - \btheta)| \\
    & \leq \norm{\bx_{k}}_{\bV_t^{-1}} \norm{\Hat{\btheta}_t - \btheta}_{\bV_t}\\
    & \leq  L_2 \sqrt{d \log(\frac{1 + tL^{2}/\lambda}{\alpha})} + \lambda^{1/2} S_2) \norm{\bx_{k}}_{\bV_t^{-1}}
    \quad \text{(lemma \ref{appendix: Ellipsoid})}
\end{align}
\end{subequations}
That is, let $c_{1}(t, k):= L_2 \sqrt{d \log(\frac{1 + tL^{2}/\lambda}{\alpha_{k}})} + \lambda^{1/2} S_2$, 
\begin{equation}
\begin{aligned}
   \BP(E_{t,k}) \geq 1 - \alpha_k
\end{aligned}
\end{equation}
Therefore,
\begin{equation}
\begin{aligned}
   \BP(\Bar{E}_t) 
   = \BP(\bigcup_{k=1}^K \Bar{E}_{t,k})
   \leq \sum_{k=1}^K \alpha_k
\end{aligned}
\end{equation}
\end{proof}

\subsection{Proof of Lemma \ref{lemma: resampling_concentration}}
\label{appendix: proof_lemma_resampling_concentration}
\begin{proof}
Recall the our definition of event $E^{\prime}_{t,k}$ and $RSS_{k,t}$,
\begin{equation*}
\begin{aligned}
    E^{\prime}_{t,k} 
    & := 
    \{
    |\Tilde{\mu}_{k,t} - \Hat{\mu}_{k,t}| 
    \leq
    c_{2}(t, k)\norm{\bx_{k}}_{\bV_{t}^{-1}}
    \} \\
    RSS_{k,t}
    & :=
    \sum_{i=1}^{s_{k,t-1}} e_{k,t,i}^2
\end{aligned}
\end{equation*}
Then control the probability of the bad event $\Bar{E}^{\prime}_{t,k}$ which indicates a "large" deviation between estimated mean and Bootstrapped mean of the $k$-th arm at round $t$. That is, $\forall t \geq K + 1, \forall k \in [K]$,
\begin{subequations}
\begin{align}
    \BP_{t}(\Bar{E}^{\prime}_{t,k})
    &=
    \BP_{t} (|\Tilde{\mu}_{k,t} - \Hat{\mu}_{k,t}| 
    >
    c_{2}(t, k)\norm{\bx_{k}}_{\bV_{t}^{-1}} )\\
    &= 
    \BP_{t} (|\frac{1}{s_{k,t-1}} \sum_{i=1}^{s_{k,t-1}} \omega_{k,t,i} e_{k,t,i}| 
    >
    c_{2}(t, k)\norm{\bx_{k}}_{\bV_{t}^{-1}} )\\
    &=
    \BP_{t} (|\sqrt{\frac{\sigma_{\omega}^2 \sum_{i=1}^{s_{k,t-1}} e_{k,t,i}^2}{s_{k,t-1}^2}} Z| 
    >
    c_{2}(t, k)\norm{\bx_{k}}_{\bV_{t}^{-1}} )\\
    &=
    \BP_{t} (|Z| 
    >
    \frac{c_{2}(t, k) s_{k,t-1} \norm{\bx_{k}}_{\bV_{t}^{-1}} }{\sqrt{\sigma_{\omega}^2 RSS_{k,t}}}
     )
    \quad  (\text{Define } Z \sim N(0,1))\\
    & \leq
    \BP_{t} (|Z| 
    >
    \frac{ c_{2}(t, k) s_{k,t-1} \norm{\bx_{k}}_{\bV_{t}^{-1}} }{\sqrt{\sigma_{\omega}^2 RSS_{k,t}}}
     )\\
    &\leq
    2\exp(-\frac{c_{2}^2(t, k) s_{k,t-1}^2 \norm{\bx_{k}}_{\bV_{t}^{-1}}^2}
    {\sigma_{\omega}^2 RSS_{k,t}})
    \quad \text{($Z$ is subgaussian with constant 1)}
\end{align}
\end{subequations}
Now let 
$\beta_{k} := 2\exp(-\frac{c_{2}^2(t, k) s_{k,t-1}^2 \norm{\bx_{k}}_{\bV_{t}^{-1}}^2}
{\sigma_{\omega}^2 RSS_{k,t}})$
then
\begin{equation}
\begin{aligned}
    c_{2}(t, k):=
    \sqrt{
    \dfrac{2 \sigma_{\omega}^{2} RSS_{k,t} \log(\frac{2}{\beta_{k}})}
    { s_{k,t-1}^2 \norm{\bx_{k}}_{\bV_{t}^{-1}}^2 }
    } 
\end{aligned}
\end{equation}
Therefore,
\begin{equation}
\begin{aligned}
    \BP_{t} (|\Tilde{\mu}_{k,t} - \Hat{\mu}_{k,t}| 
    \leq
    c_{2}(t, k)\norm{\bx_{k}}_{\bV_{t}^{-1}})
    \geq
    1 - \beta_{k}
\end{aligned}
\end{equation}

\end{proof}

\subsection{Proof of Lemma \ref{lemma: anti_concentration}}
\label{appendix: proof_lemma_anti_concentration}
\begin{proof}
Follow the same notations in \ref{appendix: proof_lemma_resampling_concentration},
$$
RSS_{k,t}
:=
\sum_{i=1}^{s_{k,t-1}} e_{k,t,i}^2
\qquad
Z \sim N(0,1)
$$
Similar to lemma 10 in \citep{wang2020residual}, the vanilla Gaussian tail lower bound, lemma \ref{appendix: Gaussian_Tail}, is used.  That is, $\forall t$, $\forall b > 0 $
\begin{subequations}
\begin{align}
    \BP_{t}(E_{t}^{\prime \prime})
    &=
    \BP_{t} (\Tilde{\mu}_{1,t} - \Hat{\mu}_{1,t}
    >
    c_{1}(t, 1)\norm{\bx_{1}}_{\bV_{t}^{-1}})\\
    &= 
    \BP_{t} ( \frac{1}{s_{1,t-1}} \sum_{i=1}^{s_{1,t-1}} \omega_{1,i} e_{1,t,i}
    >
    c_{1}(t, 1)\norm{\bx_{1}}_{\bV_{t}^{-1}} )\\
    &=
    \BP_{t} ( Z 
    >
    \frac{ c_{1}(t, 1) s_{1,t-1}\norm{\bx_{1}}_{\bV_{t}^{-1}} }{\sqrt{\sigma_{\omega}^2 RSS_{1,t}}}
     )\\
    & \geq
    \begin{cases}
    \frac{b}{\sqrt{2 \pi}} 
    \exp(-\frac{3  c_{1}^2(t,1) s_{1, t-1}^2 \norm{\bx_{1}}_{\bV_{t}^{-1}}^2}{ 2 \sigma_{\omega}^2 RSS_{1,t}}) & 
    \text{if  }
    \frac{c_{1}(t,1)  s_{1, t-1} \norm{\bx_{1}}_{\bV_t^{-1}} }{\sqrt{\sigma_{\omega}^2 RSS_{1,t}}} \geq b\\
    \Phi(-b) 
    & \text{if  } 
    0 < \frac{c_{1}(t,1)  s_{1, t-1} \norm{\bx_{1}}_{\bV_t^{-1}} }{\sqrt{\sigma_{\omega}^2 RSS_{1,t}}} < b
    \end{cases}
\end{align}
\end{subequations}
Where $b$ is the constant chosen by us. This $b$ controlling the sharpness of the lower bound of Gaussian tail. Notice that \eqref{eq:Samp-Boot_ratio} is equivalent to the condition $ \frac{c_{1}(t,1)  s_{1, t-1} \norm{\bx_{1}}_{\bV_t^{-1}} }{\sqrt{\sigma_{\omega}^2 RSS_{1,t}}} \geq b$ by the definition \eqref{eq:sample_OED} and \eqref{eq:bootstrap_OED}, the above lower bound can be writed as,
\begin{equation}
\begin{aligned}
    \BP_{t}(E_{t}^{\prime \prime})
    \geq
    \begin{cases}
    \frac{b}{\sqrt{2 \pi}} 
    \exp(-\frac{3  c_{1}^2(t,1) s_{1, t-1}^2 \norm{\bx_{1}}_{\bV_{t}^{-1}}^2}{ 2 \sigma_{\omega}^2 RSS_{1,t}})& 
    \text{if  }
    \frac{c_{1}(t, 1)}{c_{2}(t, 1)} \ge b \sqrt{2 \log \left(\frac{2}{\beta_{1}}\right)}
    \\
    \Phi(-b) 
    & \text{if  } 
    \frac{c_{1}(t, 1)}{c_{2}(t, 1)} < b \sqrt{2 \log \left(\frac{2}{\beta_{1}}\right)}
    \end{cases}
\end{aligned}
\end{equation}

\end{proof}

\subsection{Proof of Lemma \ref{lemma: connection_concentrations}}
\label{appendix: proof_lemma_connection_concentrations}
\begin{proof}
Recall our true model:
\begin{equation*}
\begin{aligned}
    \bY_{t} = \bX_{t} \btheta + \bepsilon_{t}
\end{aligned}
\end{equation*}
Further define matrix $\bQ_{k,t}$ which indicates the RSS decomposition for the $k$-th arm at time $t$:
\begin{equation}
\begin{aligned}
\label{appendix: def Q}
    [\bQ_{k,t}]_{ij} = 
    \begin{cases}
    1 & i=j \text{ and } I_{i} = k\\
    0 & \text{ otherwise }
    \end{cases}
    \forall i, j \in [t]
\end{aligned}
\end{equation}
In this proof, we will start from stating lemmas and technical condition, then give main proof which has three steps.  
\begin{lemma}
\label{appendix: lemma_RSS}
By (\ref{appendix: def Q}), which is definition of $\bQ_{k,t}$, $RSS_{t}$ can be decomposed by arms,
\begin{equation}
\begin{aligned}
    RSS_{t}:= \norm{\bY_{t} - \bX_{t} \Hat{\btheta}_{t}}_{2}^{2} = \sum_{k=1}^{K} RSS_{k,t}
\end{aligned}
\end{equation}
And $RSS_{k,t}:= \norm{\bQ_{k,t}(\bY_{t} - \bX_{t} \Hat{\btheta}_{t})}_{2}^{2}$ can be re-writed as: 
\begin{equation}
\begin{aligned}
    RSS_{k, t} = 
    & \norm{\bQ_{k,t-1} (\bI - \bX_{t-1} \bV_t^{-1}\bX_{t-1}^{\top})\bX_{t-1} \btheta}_2^2 \\
    +
    & \norm{\bQ_{k,t-1} (\bI - \bX_{t-1} \bV_t^{-1}\bX_{t-1}^{\top})\bepsilon_{t-1}}_2^2 \\
    + 
    &
    2 \btheta^{\top} \bX_{t-1}^{\top} (\bI - \bX_{t-1} \bV_t^{-1}\bX_{t-1}^{\top}) \bQ_{k,t-1}^{\top} \bQ_{k,t-1} (\bI - \bX_{t-1} \bV_t^{-1}\bX_{t-1}^{\top}) \bepsilon_{t-1}
\end{aligned}
\end{equation}
\end{lemma}
\begin{proof}
See appendix \ref{appendix: proof_lemma_RSS}.
\end{proof}
\textbf{Remark.} Lemma \ref{appendix: lemma_RSS} provides a decomposition of $RSS$ for arm $k$ at round $t$. 

\begin{lemma}
\label{appendix: lemma_two_side_bound}
Stochastic process $\{\epsilon_{t}\}_{t=1}^{\infty}$ satisfies that for some $R_{1}, R_{2} > 0$,
$$e^{R_{1} \eta^2} \leq \BE[e^{\eta \epsilon_{t}} | \mF_{t-1}] \leq e^{R_{2} \eta^2} \quad \forall \eta \geq 0$$
Singular value decomposition of $\bX_{K}$ and definition of ridge shrinkage context matrix $\bZ$ are
\begin{equation*}
\begin{aligned}
\bX_{K} &:=\bG \bSigma \bU \\
\bOmega &:= \bSigma (\bSigma^{\top} \bSigma + \lambda \bI)^{-1} \bSigma^{\top}\\
\bZ     &:= \bG \bOmega \bSigma \bU
\end{aligned}
\end{equation*}
Let $\bz_{1}$ be the vector of the first row of matrix $\bZ$ and suppose $(\bx_1^{\top} - \bz_1^{\top} \btheta)^2 \geq S_1^2$. Then $\forall \eta \geq 0$, $\forall t \geq K+1$,
\begin{equation}
\begin{aligned}
    exp(\frac{\lambda^2}{(\sigma_{\max}^2 + \lambda)^2} 
    S_1^2 L_1 \eta^2)
    \leq 
    \BE[e^{\eta \xi_{t}}]
    \leq 
    exp(\sigma_{\max}^2 S_2^2 L_2 \eta^2 )
\end{aligned}
\end{equation}
Where $\xi_{t}:= \frac{1}{\sqrt{s_{1, t-1}}}
\btheta^{\top} \bX_{t-1}^{\top} (\bI - \bX_{t-1} \bV_t^{-1}\bX_{t-1}^{\top}) \bQ_{1,t-1}^{\top} \bQ_{1,t-1} (\bI - \bX_{t-1} \bV_t^{-1}\bX_{t-1}^{\top}) \bepsilon_{t-1}$
\end{lemma}
\begin{proof}
See appendix \ref{appendix: proof_lemma_two_side_bound}.
\end{proof}
\textbf{Remark.} Lemma \ref{appendix: lemma_two_side_bound} indicates that the random variable $\xi_{t}$ which is based on noise process $\{\epsilon_{\tau}\}_{\tau=1}^{t-1}$ also has the clipping noise property. Thus this random variable is also subgaussian. This result supports our application of Lemma \ref{appendix: lemma_subG} which is given in the next part.

\begin{lemma}
\label{appendix: lemma_subG}
Suppose $X$ is a random variable such that $\exists R_{1},R_{2} > 0 $
\begin{equation}
\begin{aligned}
\exp(R_{1}t^{2}) \leq \BE[e^{tX}] \leq \exp(R_{2}t^{2}) \quad \forall t \geq 0 
\end{aligned}
\end{equation}
Then
\begin{equation}
\begin{aligned}
\BP(X \geq x) \geq C_{1} \exp(-C_{2} x^{2})
\end{aligned}
\end{equation}
Where $C_{1}:= (e-1)^{2} e^{\frac{8R_{2}}{R_{1}} - 6}$ and $C_{2}:= \frac{4 R_{2} - 2 R_{1}}{ R_{1}^2}$
\end{lemma}
\begin{proof}
See appendix \ref{appendix: proof_lemma_subG}
\end{proof}
\textbf{Remark.} This Lemma is inspired by the Theorem 1 and its proof in \citep{zhang2020non}. This Lemma gives the lower tail bound of random variable $X$ and the only condition is that there is upper and lower bound of the form $e^{Ct^2}$ for the moment generating function of $X$. 

\textbf{Technical Condition}. 
The difference between $\BP_{t}(E_{t}^{\prime \prime})$ and $\BP_{t}(\Bar{E}_{t}^{\prime})$ plays a key role in bounding regret when applying the stochastic exploration on least squared framework. The following part is the probabilistic analysis of lower bound of this difference, which will be denoted as $D < \BP_{t}(E_{t}^{\prime \prime}) - \BP_{t}(\Bar{E}_{t}^{\prime})$ in this proof. First impose some requirements on the tuning parameters $\beta, D, b$:
\begin{equation}
 \label{appendix: D_req}
\begin{aligned}
    D + \beta < \min(\Phi(-b), \frac{b}{\sqrt{2\pi}} e^{-\frac{3}{2}b^{2}})
\end{aligned}
\end{equation}
This requirement indicates three results:
\begin{align}
    & D + \beta < \Phi(-b) 
    \label{appendix: D_req_1}\\
    & D + \beta < \frac{b}{\sqrt{2\pi}} 
    \label{appendix: D_req_2}\\
    & -\frac{3}{2 \log(\frac{\sqrt{2 \pi}}{ b}(D + \beta))} < \frac{1}{b^2}
    \label{appendix: D_req_3}
\end{align}
\textbf{Main proof of lemma \ref{lemma: connection_concentrations}}\\
\textbf{Step 1: Express event $\{\BP_{t}(E_{t}^{\prime \prime}) - \BP_{t}(\Bar{E}_{t}^{\prime}) > D\}$ as an inequality of $RSS_{1,t}$}\\
The idea in this step is starting from decomposing our target event $\{\BP_{t}(E{t}^{\prime \prime}) - \BP_{t}(\Bar{E}_{t}^{\prime}) > D\}$ by the condition mentioned in lemma \ref{lemma: anti_concentration}. That is,
\begin{subequations}
\begin{align}
    & \BP(\BP_{t}(E_{t}^{\prime \prime}) - \BP_{t}(\Bar{E}_{t}^{\prime}) > D) \\
    \geq &
    \BP(\BP_{t}(E_{t}^{\prime \prime}) > D + \beta) 
    \quad  (\text{by lemma \ref{lemma: resampling_concentration}})\\
    = & 
    \BP(
    \{\BP_{t}(E_{t}^{\prime \prime}) > D + \beta\} 
    \cap
    \{ \frac{c_{1}(t,1) s_{1, t-1}\norm{\bx_{1}}_{\bV_t^{-1}} }{\sqrt{\sigma_{\omega}^2 RSS_{1,t}}} \geq b\}
    ) \notag \\
    & + 
    \BP(
    \{\BP_{t}(E_{t}^{\prime \prime}) > D + \beta\} 
    \cap
    \{ \frac{c_{1}(t,1) s_{1, t-1} \norm{\bx_{1}}_{\bV_t^{-1}} }{\sqrt{\sigma_{\omega}^2 RSS_{1,t}}} < b\}
    ) \\
    \geq &
    \BP(
    \{\frac{b}{\sqrt{2 \pi}} \exp(-\frac{3 s_{1,t-1}^2 c_{1}^2(t,1)  \norm{\bx_{1}}_{\bV_{t}^{-1}}^2}{ 2 \sigma_{\omega}^2 RSS_{1,t}}) > D + \beta\}
    \cap
    \{ \frac{c_{1}(t,1) s_{1, t-1} \norm{\bx_{1}}_{\bV_t^{-1}} }{\sqrt{\sigma_{\omega}^2 RSS_{1,t}}} \geq b\} 
    ) \notag \\
    & + 
    \BP(
    \{ \Phi(-b) > D + \beta\} 
    \cap
    \{ \frac{c_{1}(t,1) s_{1, t-1} \norm{\bx_{1}}_{\bV_t^{-1}} }{\sqrt{\sigma_{\omega}^2 RSS_{1,t}}} < b\}
    )  
    \quad  (\text{by lemma \ref{lemma: anti_concentration}})
\end{align}
\end{subequations}
Then we apply the technical condition described in \ref{appendix: D_req}, 
\begin{subequations}
\begin{align}
    & \BP(\BP_{t}(E_{t}^{\prime \prime}) - \BP_{t}(\Bar{E}_{t}^{\prime}) > D) \\
    \geq &
    \BP(
    \{ RSS_{1,t} > -\frac{3  c_{1}^{2}(t,1) s_{1, t-1}^2 \norm{\bx_{1}}^{2}_{\bV_{t}^{-1}}}{2  \sigma^{2}_{\omega} \log(\frac{\sqrt{2 \pi}}{ b}(D+\beta))}\}
    \cap
    \{ RSS_{1,t} \leq \frac{ c_{1}^{2}(t,1) s_{1, t-1}^2 \norm{\bx_{1}}^{2}_{\bV_t^{-1}}}{\sigma_{\omega}^2 b^2 } \} 
    ) \notag \\
    & + 
    \BP(RSS_{1,t} > \frac{ c_{1}^{2}(t,1) s_{1, t-1}^2 \norm{\bx_{1}}^{2}_{\bV_t^{-1}}}{\sigma_{\omega}^2 b^2 }) 
    \quad  (\text{by (\ref{appendix: D_req_1}) and (\ref{appendix: D_req_2})})\\
    = & 
    \BP(
    RSS_{1,t} > -\frac{3 c_{1}^{2}(t,1) s_{1, t-1}^2 \norm{\bx_{1}}^{2}_{\bV_{t}^{-1}}}{2 \sigma^{2}_{\omega} \log(\frac{\sqrt{2 \pi}}{ b}(D+\beta))})
    \quad (\text{by (\ref{appendix: D_req_3})})
\end{align}
\end{subequations}
\textbf{Step 2: Apply lemmas to give lower bounds}\\
In this step, three lemmas are used. 
\begin{subequations}
\begin{align}
    &     \BP(
    RSS_{1,t} > -\frac{3 c_{1}^{2}(t,1) s_{1, t-1}^2 \norm{\bx_{1}}^{2}_{\bV_{t}^{-1}}}{2 \sigma^{2}_{\omega} \log(\frac{\sqrt{2 \pi}}{ b}(D+\beta))}) \\
    \leq & 
    \BP(
    \btheta^{\top} \bX_{t-1}^{\top} (\bI - \bX_{t-1} \bV_t^{-1}\bX_{t-1}^{\top}) \bQ_{1,t-1}^{\top} \bQ_{1,t-1} (\bI - \bX_{t-1} \bV_t^{-1}\bX_{t-1}^{\top}) \bepsilon_{t-1}
    \notag \\
    &
    \qquad > \frac{3  c_{1}^{2}(t,1) s_{1, t-1}^2 \norm{\bx_{1}}^{2}_{\bV_{t}^{-1}}}
    {8 \sigma^{2}_{\omega}  \log(\frac{b}{\sqrt{2 \pi} (D + \beta)})})
    \quad (\text{by (\ref{appendix: proof_lemma_5.4_a})})
\end{align}
\end{subequations}
Where (\ref{appendix: proof_lemma_5.4_a}) is derived directly from lemma \ref{appendix: lemma_RSS}, 
\begin{equation}
\begin{aligned}
\label{appendix: proof_lemma_5.4_a}
    RSS_{1,t} \geq 
    4 \btheta^{\top} \bX_{t-1}^{\top} (\bI - \bX_{t-1} \bV_t^{-1}\bX_{t-1}^{\top}) \bQ_{1,t-1}^{\top} \bQ_{1,t-1} (\bI - \bX_{t-1} \bV_t^{-1}\bX_{t-1}^{\top}) \bepsilon_{t-1}
\end{aligned}
\end{equation}
Denote $\xi_{t}:= \frac{1}{\sqrt{s_{1, t-1}}}
\btheta^{\top} \bX_{t-1}^{\top} (\bI - \bX_{t-1} \bV_t^{-1}\bX_{t-1}^{\top}) \bQ_{1,t-1}^{\top} \bQ_{1,t-1} (\bI - \bX_{t-1} \bV_t^{-1}\bX_{t-1}^{\top}) \bepsilon_{t-1}$. By lemma \ref{appendix: lemma_two_side_bound}, moment generating function of random variable $\xi_{t}$ has upper bound and lower bound,
\begin{equation*}
\begin{aligned}
    exp(\frac{\lambda^2}{(\sigma_{\max}^2 + \lambda)^2} 
    S_1^2 L_1 \eta^2)
    \leq 
    \BE[e^{\eta \xi_{t}}]
    \leq 
    exp(\sigma_{\max}^2 S_2^2 L_2 \eta^2 )
\end{aligned}
\end{equation*}
Then applying lemma \ref{appendix: lemma_subG}, 
\begin{subequations}
\begin{align}
    \BP(\BP_{t}(E_{t}^{\prime \prime}) - \BP_{t}(\Bar{E}_{t}^{\prime}) > D) 
    \geq & \BP(
    RSS_{1,t} > -\frac{3  c_{1}^{2}(t,1) s_{1, t-1}^2 \norm{\bx_{1}}^{2}_{\bV_{t}^{-1}}}{2 \sigma^{2}_{\omega} \log(\frac{\sqrt{2 \pi}}{b}(D+\beta))}) \\
    \geq &
    \BP(\xi_{t}  > \frac{3c_{1}^{2}(t,1) s_{1, t-1}^{3/2} \norm{\bx_{1}}^{2}_{\bV_{t}^{-1}} }
    {8 \sigma^{2}_{\omega}  \log(\frac{b}{\sqrt{2 \pi} (D + \beta)})}) \\
    \geq & M_{1} \exp(-M_{2} (\frac{3c_{1}^{2}(t,1) s_{1, t-1}^{3/2} \norm{\bx_{1}}^{2}_{\bV_{t}^{-1}} }
    {8 \sigma^{2}_{\omega}  \log(\frac{b}{\sqrt{2 \pi} (D + \beta)})})^2)
\end{align}
\end{subequations}
Where 
\begin{equation}
\label{appendix: proof_lemma_connection_concentrations_M1}
    M_1 :=
    (e-1)^2 \exp(\dfrac{8\sigma_{max}^2 S_2^2 L_2}{\frac{\lambda^2}{(\sigma_{max}^2 + \lambda)^2} S_1^2 L_1} - 6)
\end{equation}
\begin{equation}
\label{appendix: proof_lemma_connection_concentrations_M2}
    M_2 :=
    \dfrac{4\sigma_{\max}^2 S_2^2 L_2 - 2 \frac{\lambda^2}{(\sigma_{\max}^2 + \lambda)^2} S_1^2 L_1}{(\frac{\lambda^2}{(\sigma_{\max}^2 + \lambda)^2} S_1^2 L_1)^2}
\end{equation}
Let 
$1-\gamma
:= M_{1} \exp(-M_{2} (\frac{3c_{1}^{2}(t,1) s_{1, t-1}^{3/2} \norm{\bx_{1}}^{2}_{\bV_{t}^{-1}} }
{8 \sigma^{2}_{\omega}  \log(\frac{b}{\sqrt{2 \pi} (D + \beta)})})^2)$,
then
\begin{equation}
\begin{aligned}
    D: = \dfrac{b}{\sqrt{2 \pi}} \exp(-\frac{3c_{1}^{2}(t,1) s_{1, t-1}^{3/2} \norm{\bx_{1}}^{2}_{\bV_{t}^{-1}} }{8 \sigma^{2}_{\omega} \sqrt{\frac{1}{M_2} \log(\frac{M_1}{1 - \gamma})}}) -\beta
\end{aligned}
\end{equation}
Thus the connection between concentration and anti-concentration can be described as the following high probability lower bound,
\begin{equation}
\begin{aligned}
\BP(\BP_{t}(E_{t}^{\prime \prime}) - \BP_{t}(\Bar{E}_{t}^{\prime}) 
> 
\dfrac{b}{\sqrt{2 \pi}} \exp(-\frac{3c_{1}^{2}(t,1) s_{1, t-1}^{3/2} \norm{\bx_{1}}^{2}_{\bV_{t}^{-1}} }{8 \sigma^{2}_{\omega} \sqrt{\frac{1}{M_2} \log(\frac{M_1}{1 - \gamma})}}) -\beta
) \geq 1-\gamma
\end{aligned}
\end{equation}
Notice that $ \norm{\bx_{1}}^{2}_{\bV_{t}^{-1}} \leq  \frac{\norm{\bx_1}_2^2}{\sigma_{\min}^2 + \lambda}$, then $\forall t \geq K + 1$, with probability at least $1- \gamma$,
\begin{equation}
\begin{aligned}
\BP_{t}(E_{t}^{\prime \prime}) - \BP_{t}(\Bar{E}_{t}^{\prime}) 
> 
\dfrac{b}{\sqrt{2 \pi}} \exp( - \frac{3 s_{1, t-1}^{3/2} c_{1}^{2}(t,1) \norm{\bx_{1}}^{2}_2 }{8 \sigma^{2}_{\omega} (\sigma_{\min}^2 + \lambda) \sqrt{\frac{1}{M_2} \log(\frac{M_1}{1 - \gamma})}}) -\beta
\end{aligned}
\end{equation}
Where $M_1$, $M_2$ are defined as (\ref{appendix: proof_lemma_connection_concentrations_M1}) and (\ref{appendix: proof_lemma_connection_concentrations_M2}). 

Technical condition on $b$ becomes,
\begin{equation}
\label{requirement_for_b}
\begin{aligned}
    \dfrac{b}{\sqrt{2 \pi}} 
    \exp( - \frac{3 s_{1, t-1}^{3/2} c_{1}^{2}(t,1) \norm{\bx_{1}}^{2}_2 }{8 \sigma^{2}_{\omega} (\sigma_{\min}^2 + \lambda) \sqrt{\frac{1}{M_2} \log(\frac{M_1}{1 - \gamma})}}) 
    < 
    min(\Phi(-b), \frac{b}{\sqrt{2\pi}} e^{-\frac{3}{2}b^{2}})
\end{aligned}
\end{equation}
\end{proof}

\clearpage

\section{Proofs of Technical Lemmas}\label{appendix: lemma}

\subsection{Proof of Lemma \ref{appendix: lemma_gap}}
\label{appendix: proof_lemma_gap}
\begin{proof}
This proof is mainly adapted from proof of lemma 2 in \citep{kveton2019perturbed}. The main extension is to redefine the concept of ''least uncertain undersampled'' arm to meet the need of residual bootstrap exploration.
First define 'under sampled' arms,
\begin{equation}\label{eq:under_sampled_arms}
\begin{aligned}
    \bar{\mS}_{t} := \{k\in [K]: c_{t,k} \norm{\bx_{k}}_{\bV_{t}^{-1}} \geq \Delta_{k}\}
\end{aligned}
\end{equation}
Where $c_{t,k}: = c_{1}(t,k) + c_{2}(t,k)$ and the set of "sufficiently sampled" arms is $\mS_{t} := [K]\setminus \bar{\mS}_{t}$. Also define the "least uncertain" arm at round $t$,
\begin{equation}
\begin{aligned}
    J_{t} := \underset{k \in \bar{\mS}_{t}}{\arg\min} \mbox{ } c_{t,k} \norm{\bx_{k}}_{\bV_{t}^{-1}}
\end{aligned}
\end{equation}
Then when event $E_{t}^{\prime}$ occurs,
\begin{subequations}
\begin{align}
    \Delta_{I_{t}} 
    &= \mu_1 - \mu_{I_t} + \mu_{J_{t}} - \mu_{J_{t}} \\
    &= \Delta_{J_{t}} + \mu_{J_{t}} - \mu_{I_{t}} \\
    &= \Delta_{J_{t}} + \mu_{J_{t}} - \Tilde{\mu}_{J_{t},t} + \Tilde{\mu}_{J_{t},t} - \Tilde{\mu}_{I_{t},t} + \Tilde{\mu}_{I_{t},t} - \mu_{I_{t}} \\
    &\leq 
    \Delta_{J_{t}} + c_{t, J_t} \norm{\bx_{J_t}}_{\bV_t^{-1}} + c_{t, I_t} \norm{\bx_{I_t}}_{\bV_t^{-1}} + \Tilde{\mu}_{J_{t},t} - \Tilde{\mu}_{I_{t},t} 
    \quad (E_{t} \cap E_t^{\prime})\\
    & \leq 
    \Delta_{J_{t}} + c_{t, J_t} \norm{\bx_{J_t}}_{\bV_t^{-1}} + c_{t, I_t} \norm{\bx_{I_t}}_{\bV_t^{-1}} 
    \quad (\Tilde{\mu}_{J_{t},t} < \Tilde{\mu}_{I_{t},t})\\
    & \leq 
    2 c_{t, J_t} \norm{\bx_{J_t}}_{\bV_t^{-1}} + c_{t, I_t} \norm{\bx_{I_t}}_{\bV_t^{-1}}
    \quad (J_t \in \bar{\mS}_{t})
\end{align}
\end{subequations}
Thus conditional expected gap can be bounding by the norms of two special arms $I_t$ and $J_t$ at round $t$,
\begin{subequations}
\begin{align}
    \BE_{t}[\Delta_{I_{t}} ]
    &= \BE_{t}[\Delta_{I_{t}} \BI\{E_{t}^{\prime}\}] + \BE_{t}[\Delta_{I_{t}} \BI\{\Bar{E}_{t}^{\prime}\}] \\
    & \leq
    \BE_{t}[ 2 c_{t, J_t} \norm{\bx_{J_t}}_{\bV_t^{-1}} + c_{t, I_t} \norm{\bx_{I_t}}_{\bV_t^{-1}} ] + M\BP_{t}(\Bar{E}_{t}^{\prime})
\end{align}
\end{subequations}
Now we need to bound the norm of $J_t$ by the norm of $I_t$. The key observation to find the relation between $I_t$ and $J_t$ is
\begin{equation}
\begin{aligned}
    \BE_{t}[ c_{t, I_t}\norm{\bx_{I_t}}_{\bV_t^{-1}} ]
    \geq 
    \BE_{t}[c_{t, I_t}\norm{\bx_{I_t}}_{\bV_t^{-1}} | I_t \in \bar{\mS}_{t}] \BP_{t}(I_t \in \bar{\mS}_{t})
    \geq
    c_{t, J_t} \norm{\bx_{J_t}}_{\bV_t^{-1}} \BP_{t}(I_t \in \bar{\mS}_{t})
\end{aligned}
\end{equation}
Thus
\begin{equation}
\begin{aligned}
    c_{t, J_t} \norm{\bx_{J_t}}_{\bV_t^{-1}} 
    \leq \dfrac{\BE_{t}[ c_{t, I_t}\norm{\bx_{I_t}}_{\bV_t^{-1}} ]}{\BP_{t}(I_t \in \bar{\mS}_{t})}
\end{aligned}
\end{equation}
Now we need to give lower bound of $\BP_{t}(I_t \in \bar{\mS}_{t})$,
\begin{subequations}
\begin{align}
    \BP_{t}(I_t \in \bar{\mS}_{t}) 
    &= 
    \BP_{t}(\exists k \in \bar{\mS}_{t} \text{ s.t } \Tilde{\mu}_{k,t} > \underset{j \in \mS_t}{\max} \mbox{ } \Tilde{\mu}_{j,t}) \\
    &\geq
    \BP_{t}(\Tilde{\mu}_{1,t} > \underset{j \in \mS_t}{\max} \mbox{ } \Tilde{\mu}_{j,t}) 
    \quad (1 \in \bar{\mS}_{t})\\
    &\geq 
    \BP_{t}(\{\Tilde{\mu}_{1,t} > \underset{j \in \mS_t}{\max} \mbox{ }  \Tilde{\mu}_{j,t}\} \cap E_{t}^{\prime}) \\
    & \geq
    \BP_{t}(\{\Tilde{\mu}_{1,t} > \mu_1\} \cap E_{t}^{\prime}) 
    \quad (\text{by (\ref{appendix: proof_lemma_gap_a})})\\
    &\geq
    \BP_{t}(\Tilde{\mu}_{1,t} > \mu_1) - \BP_{t}( \Bar{E}_{t}^{\prime}) \\
    & \geq
    \BP_{t}(E_{t}^{\prime \prime}) - \BP_{t}( \Bar{E}_{t}^{\prime})
    \quad (\text{by (\ref{appendix: proof_lemma_gap_b})})
\end{align}
\end{subequations}
Where (\ref{appendix: proof_lemma_gap_a}), (\ref{appendix: proof_lemma_gap_b}) are
\begin{equation}
\begin{aligned}
\label{appendix: proof_lemma_gap_a}
    \forall j \in \mS_t \quad & \Tilde{\mu}_{j,t} \leq \mu_j + c_{t,j} \norm{\bx_{j}}_{\bV_t^{-1}} < \mu_j + \Delta_j = \mu \\
    \Rightarrow
    & \{\Tilde{\mu}_{1,t} > \mu_1\} \subset \{\Tilde{\mu}_{1,t} > \Tilde{\mu}_{j,t} \quad \forall j \in \mS_t\}
\end{aligned}
\end{equation}

\begin{equation}
\begin{aligned}
\label{appendix: proof_lemma_gap_b}
    \{\Tilde{\mu}_{1,t} - \Hat{\mu}_{1,t} >c_{1}(t, 1) \norm{\bx_{1}}_{\bV_t^{-1}} \} \subset \{\Tilde{\mu}_{1,t} > \mu_1\} \quad \text{(since $E_{t}$ occurs)}
\end{aligned}
\end{equation}
Therefore,
\begin{equation}
\begin{aligned}
    \BE_{t}[\Delta_{I_{t}}] \leq (\dfrac{2}{\BP_{t}(E_{t}^{\prime \prime}) - \BP_{t}(\Bar{E}_{t}^{\prime})}+1)
    (c_{1}(t,I_t)+c_{2}(t,I_t))\BE_{t}[\norm{\bx_{I_t}}_{\bV_{t}^{-1}}] + M\BP(\Bar{E}^{\prime}_{t})
\end{aligned}
\end{equation}
\end{proof}

\clearpage

\subsection{Proof of Lemma \ref{appendix: lemma_lm_RSS}}
\label{appendix: proof_lemma_lm_RSS}

\begin{proof}

First define $\{\epsilon_{I_t,i}\}_{i=1}^{s_{I_t, t-1}}$ for the noise of arm $I_t$ at round $t$. Note that these $\{\epsilon_{I_t,i}\}_{i=1}^{s_{I_t, t-1}}$ is a subset of the noise vector $\bepsilon_{t-1} = (\epsilon_1,...,\epsilon_{t-1})^{\top}$ at round $t$. Also define $\mF_{I_t,i}$, the randomness history until the noise $\epsilon_{I_t,i}$ is generated and let $\mI_{I_t, t}$ be the set of time stamps when arm $I_t$ is pulled up to round $t$. For example, suppose arm $1$ is pulled at round $1,11,21,25$ up to round $26$, then $\mI_{1,26} = \{1,11,21,25\}$ and noise set is $\{\epsilon_{1,i}\}_{i=1}^{s_{1, 25}} = \{\epsilon_{1,1}, \epsilon_{1,2}, \epsilon_{1,3}, \epsilon_{1,4}\}$. For one of these noises such as $\epsilon_{1,3}$, $\mF_{1,3} = \mF_{20}$ since $\epsilon_{1,3} = \epsilon_{21}$, indicating $\BE[e^{\eta \epsilon_{1,3}} | \mF_{20}] \leq e^{R_{2} \eta^2} \mbox{, }\forall \eta \geq 0$. As a result, other expressions of residuals and RSS of the arm pulled at round $t \geq K+1$ are

\begin{align}
    e_{I_t,t,i} 
    &= 
    \bx_{I_t}^{\top} \btheta + \epsilon_{I_t,i} - \bx_{I_t}^{\top} \Hat{\btheta}_t\\
    RSS_{I_t,t}
    &=
    \sum_{i=1}^{s_{I_t,t-1}} e_{I_t,t,i}^2
    = 
    \sum_{i=1}^{s_{I_t,t-1}} (\bx_{I_t}^{\top} \btheta + \epsilon_{I_t,i} - \bx_{I_t}^{\top} \Hat{\btheta}_t)^2
\end{align}
Starting from ridge estimate $\Hat{\btheta}_t$,
\begin{subequations}
\begin{align}
    \Hat{\btheta}_t 
    &= 
    \bV_t^{-1} \bX_{t-1}^{\top}(\bX_{t-1} \btheta + \bepsilon_{t-1}) \\
    &= 
    \bV_t^{-1} \bX_{t-1}^{\top} \bX_{t-1} \btheta + \bV_t^{-1} \bX_{t-1}^{\top} \bepsilon_{t-1} \\    
    &= 
    \bV_t^{-1} \bX_{t-1}^{\top} \bepsilon_{t-1} + 
    \bV_t^{-1} (\bX_{t-1}^{\top} \bX_{t-1} + \lambda \bI) \btheta - \lambda \bV_t^{-1} \btheta \\
    &= 
    \bV_t^{-1} \bX_{t-1}^{\top} \bepsilon_{t-1} - \lambda \bV_t^{-1} \btheta + \btheta 
\end{align}
\end{subequations}
Thus,
\begin{subequations}
\begin{align}
    \bx_{I_t}^{\top} \btheta - \bx_{I_t}^{\top} \Hat{\btheta}_t
    &=
    \bx_{I_t}^{\top} \btheta - \bx_{I_t}^{\top} \bV_t^{-1} \bX_{t-1}^{\top} \bepsilon_{t-1} + \lambda \bx_{I_t}^{\top} \bV_t^{-1} \btheta - \bx_{I_t}^{\top} \btheta \\
    & = 
    \langle \bx_{I_t}, \bX_{t-1}^\top\bepsilon_{t-1}\rangle_{\bV_{t}^{-1}}-\lambda \langle \bx_{I_t}, \btheta\rangle_{\bV_{t}^{-1}}
\end{align}
\end{subequations}
So RSS becomes,
\begin{equation}
\begin{aligned}
    RSS_{I_t,t}
    & =
    \sum_{i=1}^{s_{I_t,t-1}}
    (\bx_{I_t}^{\top} \btheta - \bx_{I_t}^{\top} \Hat{\btheta}_t + \epsilon_{I_t,i})^2\\
    & \leq 
    2 s_{I_t,t-1} 
    \bigg(\langle \bx_{I_t}, \bX_{t-1}^\top\bepsilon_{t-1}\rangle_{\bV_{t}^{-1}}-\lambda \langle \bx_{I_t}, \btheta\rangle_{\bV_{t}^{-1}}\bigg)^2 +
    2 \sum_{i=1}^{s_{I_t,t-1}}
    \epsilon_{I_t,i}^2
\end{aligned}
\end{equation}
Therefore, 
\begin{subequations}
\begin{align}
    \sum_{t=K+1}^{n} 
    \BE[\sqrt{\dfrac{ RSS_{I_t,t}} {s_{I_t,t-1}^2}}]
    & \leq
    \sum_{t = K+1}^n \BE[\sqrt{2  \bigg(\langle \bx_{I_t}, \bX_{t-1}^\top\bepsilon_{t-1}\rangle_{\bV_{t}^{-1}}-\lambda \langle \bx_{I_t}, \btheta\rangle_{\bV_{t}^{-1}}\bigg)^2 + \frac{2}{s_{I_t,t-1}^2} \sum_{i=1}^{s_{I_t,t-1}}
    \epsilon_{I_t,i}^2}] \\
    & \leq
    \sqrt{2} \sum_{t = K+1}^n \bE_1^{(t)} + \sqrt{2} \sum_{t = K+1}^n \bE_2^{(t)}
\end{align}
\end{subequations}
where
\begin{align}
    \bE_1^{(t)}     
    & =  
    \BE[ 
    \langle \bx_{I_t}, \bX_{t-1}^\top\bepsilon_{t-1}\rangle_{\bV_{t}^{-1}}-\lambda \langle \bx_{I_t}, \btheta\rangle_{\bV_{t}^{-1}}]\\
    \bE_2^{(t)}
    & = 
    \BE[\sqrt{\frac{1}{s_{I_t,t-1}^2} \sum_{i=1}^{s_{I_t,t-1}}
    \epsilon_{I_t,i}^2}]
\end{align}
The following part is bounding $\sum_{t = K+1}^n \bE_1^{(t)}$ and $\sum_{t = K+1}^n \bE_2^{(t)}$ respectively.\\
\textbf{Bounding $\sum_{t = K+1}^n \bE_1^{(t)}$.}\\
By Cauchy-Schwarz inequality,
\begin{subequations}
\begin{align}
    \bigg(\langle \bx_{I_t}, \bX_{t-1}^\top\bepsilon_{t-1}\rangle_{\bV_{t}^{-1}}-\lambda \langle \bx_{I_t}, \btheta\rangle_{\bV_{t}^{-1}}\bigg)^2
    & \leq 
    \bigg(
    \norm{\bx_{I_t}}_{\bV_t^{-1}} \norm{\bX_{t-1}^\top\bepsilon_{t-1}}_{\bV_t^{-1}} +
    \lambda \norm{\bx_{I_t}}_{\bV_t^{-1}} \norm{\btheta}_{\bV_t^{-1}}
    \bigg)^2\\
    & \leq
   \bigg(
    \norm{\bx_{I_t}}_{\bV_t^{-1}} \norm{\bX_{t-1}^\top\bepsilon_{t-1}}_{\bV_t^{-1}} +
    \norm{\bx_{I_t}}_{\bV_t^{-1}} (\lambda^{1/2} S_{2})
    \bigg)^2
    \quad \text{(by (\ref{appendix: proof_lemma_lm_RSS_a}))}\\
    & =
    \bigg(
    \|\bx_{I_{t}}\|_{\bV_{t}^{-1}}\big(\|\bX_{t-1}^\top\bepsilon_{t-1}\|_{\bV_{t}^{-1}}+\lambda^{1/2} S_2 \big)
    \bigg)^2
\end{align}
\end{subequations}
where (\ref{appendix: proof_lemma_lm_RSS_a}) is
\begin{equation}
\label{appendix: proof_lemma_lm_RSS_a}
\begin{aligned}
    \norm{\btheta}_{\bV_t^{-1}}^2 
    \leq \lambda_{max}(\bV_t^{-1})  \norm{\btheta}_{2}^2 
    = \frac{1}{\lambda } \norm{\btheta}_{2}^2
    \leq \frac{1}{\lambda } S_{2}^2
\end{aligned}
\end{equation}
By lemma \ref{appendix: self_normalized}, with probability at least $1-\delta$, 
\begin{subequations}
\begin{align}
    \bigg(\langle \bx_{I_t}, \bX_{t-1}^\top\bepsilon_{t-1}\rangle_{\bV_{t}^{-1}}-\lambda \langle \bx_{I_t}, \btheta\rangle_{\bV_{t}^{-1}}\bigg)^2
    &\leq 
    \norm{\bx_{I_t}}_{\bV_t^{-1}}^2(L_2 \sqrt{2\log(\frac{\det(\bV_{t})^{1/2}\det(\lambda \bI)^{-1/2}}{\delta})} + \lambda^{1/2} S_2)^2 \\
    & = 
    \norm{\bx_{I_t}}_{\bV_t^{-1}}^2 (L_2 \sqrt{
    2\log(\frac{(\lambda^{d-r} \prod_{j=1}^r (\sigma_j^2 + \lambda) )^{1/2} \lambda^{-d/2}}{\delta})
    } + \lambda^{1/2} S_2)^2  \\
    & \leq
    \norm{\bx_{I_t}}_{\bV_t^{-1}}^2 (L_2 \sqrt{
    r log(1 + \sigma_{max}^2/\lambda) + 2\log(\frac{1}{\delta})
    } + \lambda^{1/2} S_2)^2 
\end{align}
\end{subequations}
Therefore, with probability at least $1-\delta$,
\begin{subequations}
\begin{align}
    \sum_{t = K+1}^n \bE_1^{(t)} 
    & \leq 
    (L_2 \sqrt{
    r \log(1 + \sigma_{max}^2/\lambda) + 2\log(\frac{1}{\delta})
    } + \lambda^{1/2} S_2)
    \sum_{t = K+1}^n \BE[ \norm{\bx_{I_t}}_{\bV_t^{-1}}]
\end{align}
\end{subequations}
\textbf{Bounding $\sum_{t = K+1}^n \bE_2^{(t)}$.}\\
First separate $\sum_{t = K+1}^n \bE_2^{(t)}$ by arms,
\begin{subequations}
\begin{align}
    \sum_{t = K+1}^n \bE_2^{(t)}
    & = 
    \sum_{t = K+1}^n \BE[\sqrt{\frac{1}{s_{I_t,t-1}^2} \sum_{i=1}^{s_{I_t,t-1}}
    \epsilon_{I_t,i}^2}]\\
    & \leq 
    \sum_{k=1}^K \BE[ \sum_{t \in \mI_{k,n}} \sqrt{\frac{1}{s_{k,t-1}^2} \sum_{i=1}^{s_{k,t-1}}
    \epsilon_{k,i}^2} ] \\
    & =
    \sum_{k=1}^K \BE[ \sum_{j = 1}^{s_{k,n-1}} 
    \sqrt{\frac{1}{j^2} (\epsilon_{k,1}^2 + \dots + \epsilon_{k,j}^2)}]
\end{align}
\end{subequations}
For each arm, 
\begin{subequations}
\begin{align}
    &\BE[ \sum_{j = 1}^{s_{k,n-1}} 
    \sqrt{\frac{1}{j^2} (\epsilon_{k,1}^2 + \dots + \epsilon_{k,j}^2)}] \\
    =& 
    \BE[ \sum_{j = 1}^{s_{k,n-1}} 
    \BE[\sqrt{\frac{1}{j^2} (\epsilon_{k,1}^2 + \dots + \epsilon_{k,j}^2)}|\mF_{k,j}]] \\
    \leq&
    \BE[ \sum_{j = 1}^{s_{k,n-1}} 
    \sqrt{\BE[\frac{1}{j^2} (\epsilon_{k,1}^2 + \dots + \epsilon_{k,j}^2)|\mF_{k,j}]}] \\
    \leq&
    \BE[ \sum_{j = 2}^{s_{k,n-1}} 
    \sqrt{\frac{1}{j^2} (\epsilon_{k,1}^2 + \dots + \epsilon_{k,j-1}^2)+\frac{1}{j^2}4L_2} + 2\sqrt{L_2}] 
    \quad \text{(by lemma \ref{appendix: Second_Moment_SubG})}\\
    =&
    \BE[ \sum_{j = 1}^{s_{k,n-1}-1} 
    \sqrt{\frac{1}{(j+1)^2} (\epsilon_{k,1}^2 + \dots + \epsilon_{k,j}^2)+\frac{1}{(j+1)^2}4L_2} + 2\sqrt{L_2}] 
\end{align}
\end{subequations}
Conditioning on appropriate historical randomness $\mF_{k,j}$ again,
\begin{subequations}
\begin{align}
    &\BE[ \sum_{j = 1}^{s_{k,n-1}} 
    \sqrt{\frac{1}{j^2} (\epsilon_{k,1}^2 + \dots + \epsilon_{k,j}^2)}] \\
    =&
    \BE[ \sum_{j = 1}^{s_{k,n-1}-1} 
    \BE[\sqrt{\frac{1}{(j+1)^2} (\epsilon_{k,1}^2 + \dots + \epsilon_{k,j}^2)+\frac{1}{(j+1)^2}4L_2}|\mF_{k,j}] + 2\sqrt{L_2}] \\
    \leq &
    \BE[ \sum_{j = 1}^{s_{k,n-1}-1} 
    \sqrt{
    \BE[\frac{1}{(j+1)^2} (\epsilon_{k,1}^2 + \dots + \epsilon_{k,j}^2)+\frac{1}{(j+1)^2}4L_2|\mF_{k,j}]} + 2\sqrt{L_2}] \\
    \leq &
    \BE[ \sum_{j = 2}^{s_{k,n-1}-1} 
    \sqrt{
    \frac{1}{(j+1)^2} (\epsilon_{k,1}^2 + \dots + \epsilon_{k,j-1}^2)+\frac{2}{(j+1)^2}4L_2} + \frac{2}{\sqrt{2}}\sqrt{L_2} + 2\sqrt{L_2}] 
    \quad \text{(by lemma \ref{appendix: Second_Moment_SubG})}\\
    = &
    \BE[ \sum_{j = 1}^{s_{k,n-1}-2} 
    \sqrt{
    \frac{1}{(j+2)^2} (\epsilon_{k,1}^2 + \dots + \epsilon_{k,j}^2)+\frac{2}{(j+2)^2}4L_2} + (1 + \frac{1}{2}) \times 2\sqrt{L_2}]
\end{align}
\end{subequations}
Applying conditional expectation given historical randomness until there is no randomness from noise,
\begin{subequations}
\begin{align}
    \BE[ \sum_{j = 1}^{s_{k,n-1}} 
    \sqrt{\frac{1}{j^2} (\epsilon_{k,1}^2 + \dots + \epsilon_{k,j}^2)}] 
    & \leq
    2\sqrt{L_2}
    \BE[(1 + \frac{1}{2} + \dots + \frac{1}{s_{k,n-1}}) ]\\
    & \leq
    2\sqrt{L_2} \BE[\log(s_{k,n-1}) + 1] 
    \quad \text{(by (\ref{appendix: proof_lemma_lm_RSS_b}))}\\
    & \leq 
    2 \sqrt{L_2}(\log n + 1)
\end{align}
\end{subequations}
where (\ref{appendix: proof_lemma_lm_RSS_b}) is
\begin{equation}
\label{appendix: proof_lemma_lm_RSS_b}
\begin{aligned}
    \sum_{i=1}^{s_{k,n-1}} \frac{1}{i}
    \leq 1 + \int_{1}^{s_{k,n-1}} \frac{1}{u} du
    = \log(s_{k,n-1}) + 1
\end{aligned}
\end{equation}
Consequently,
\begin{equation}
\begin{aligned}
    \sum_{t = K+1}^n \bE_2^{(t)} \leq  2 K\sqrt{L_2}(\log n + 1)
\end{aligned}
\end{equation}
Therefore, with probability at least $1-\delta$,
\begin{equation}
\begin{aligned}
    \sum_{t=K+1}^{n} 
    \BE[\sqrt{\dfrac{ RSS_{I_t,t}} {s_{I_t,t-1}^2}}] 
    \leq
    \sqrt{2} (L_2 \sqrt{
    r \log(1 + \sigma_{max}^2/\lambda) + 2\log(\frac{1}{\delta})
    } + \lambda^{1/2} S_2)
    \sum_{t = K+1}^n \BE[ \norm{\bx_{I_t}}_{\bV_t^{-1}}] 
    +  2 \sqrt{2} K\sqrt{L_2}(\log n + 1)
\end{aligned}
\end{equation}
\end{proof}

\subsection{Proof of Lemma \ref{appendix: lemma_norm}}
\label{appendix: proof_lemma_norm}

\begin{proof}
Similar version of this lemma is proven by \citep{abbasi2011improved} and \citep{lattimore2020bandit}, following part is adapted version based on the notations in this paper. The main adaptation is using the eigenvalues of context matrix $\bX_{K}$ under stochastic linear bandit setting. This proof requires proof of two elementary algebraic results,

\begin{equation}
\label{appendix: proof_lemma_norm_1}
\begin{aligned}
log\dfrac{det(\bV_{n})}{det(\bV_{K+1})} = \sum_{t=K+1}^n \log(1 + \norm{\bx_{I_t}}^2_{\bV_{t}^{-1}})
\end{aligned}
\end{equation}
\begin{equation}
\label{appendix: proof_lemma_norm_2}
\begin{aligned}
\log \dfrac{det(\bV_{n})}{det(\bV_{K+1})}  \leq
d \log(\frac{ \lambda + n \sum_{i=1}^r \sigma_i^2/d}{ det(\bV_{K+1})^{1/d}})
\end{aligned}
\end{equation}
\textbf{Step 1: Proof of (\ref{appendix: proof_lemma_norm_1}).}\\
Starting from the determinant of $\bV_n$,
\begin{subequations}
\begin{align}
    det(\bV_n)
    & = det(\bV_{n-1} + \bx_{I_{n-1}} \bx_{I_{n-1}}^{\top})\\
    & = det(\bV_{n-1}^{1/2} (\bI + \bV_{n-1}^{-1/2}\bx_{I_{n-1}} \bx_{I_{n-1}}^{\top} \bV_{n-1}^{-1/2} ) \bV_{n-1}^{1/2}) \\
    & = det(\bV_{n-1}) (1 + \norm{\bx_{I_{n-1}}}_{\bV_{n-1}^{-1}}^2)\\
    & = det(\bV_{K+1}) \prod_{t=K+1}^n (1 + \norm{\bx_{I_{t-1}}}_{\bV_{t-1}^{-1}}^2)
\end{align}
\end{subequations}
Then take logarithm on both side and (\ref{appendix: proof_lemma_norm_1}) is obtained.\\
\textbf{Step 2: Proof of (\ref{appendix: proof_lemma_norm_2}).}\\
By inequality between trace and determinant and notice that eigenvalues of $\bV_{n}$ are $\sigma_1^2+\lambda,...,\sigma_r^2+\lambda$ and $d-r$ $\lambda$, then,
\begin{equation}
\begin{aligned}
    det(\bV_n)
    \leq (\dfrac{1}{d}tr(\bV_n))^{d}
    = (\dfrac{d \lambda + \sum_{i=1}^r \sigma_i^2}{d})^{d}
\end{aligned}
\end{equation}
Thus,
\begin{equation}
\begin{aligned}
    \log \dfrac{det(\bV_{n})}{det(\bV_{K+1})} 
    \leq 
    \log (\dfrac{1}{det(\bV_{K+1})} (\dfrac{d \lambda + \sum_{i=1}^r \sigma_i^2}{d})^{d}) 
    =
    d \log(\frac{ \lambda + \sum_{i=1}^r \sigma_i^2/d}{ det(\bV_{K+1})^{1/d}})
\end{aligned}
\end{equation}
\textbf{Step 3: Provide upper bound of sum of norms}\\
By (\ref{appendix: proof_lemma_norm_1}) and (\ref{appendix: proof_lemma_norm_2}), using a analytic result $x \leq 2log(1+x) \forall x\geq0$,then sum of the context norm under matrix $\bV_t^{-1}$ can be bounded,
\begin{subequations}
\begin{align}
    \sum_{t=K+1}^{n} \norm{\bx_{I_t}}_{\bV_{t}^{-1}}^2 
    &\leq
    \sum_{t=K+1}^{n} 2\log(1+\norm{\bx_{I_t}}_{\bV_{t}^{-1}}^2)\\
    &=
    2\log\dfrac{det(\bV_{n})}{det(\bV_{K+1})} 
    \quad \text{(by (\ref{appendix: proof_lemma_norm_1}))}\\
    & \leq 
    2 d \log(\frac{ \lambda + n \sum_{i=1}^r \sigma_i^2/d}{ det(\bV_{K+1})^{1/d}})
    \quad \text{(by (\ref{appendix: proof_lemma_norm_2}))}\\
    & =
    2 d \log(\frac{ \lambda + \sum_{i=1}^r \sigma_i^2/d}{ (\lambda^{d-r} \prod_{i=1}^r(\sigma_{i}^2 + \lambda))^{1/d}}) \\
    & \leq
    2 d \log(1 + \frac{ n \sum_{i=1}^r \sigma_i^2}{ d \lambda})
\end{align}
\end{subequations}
Therefore, from Cauchy-Schwarz inequality, 
\begin{equation}
\begin{aligned}
    \sum_{t=K+1}^{n} \norm{\bx_{I_t}}_{\bV_{t}^{-1}} 
    \leq 
    \sqrt{(n-K)\sum_{t=K+1}^{n} \norm{\bx_{I_t}}_{\bV_{t}^{-1}}^2}
    \leq 
    \sqrt{2(n-K) d \log(1 + \frac{ \sum_{i=1}^r \sigma_i^2}{ d \lambda})}
\end{aligned}
\end{equation}
\end{proof}

\subsection{Proof of Lemma \ref{appendix: lemma_RSS}}
\label{appendix: proof_lemma_RSS}

\begin{proof}
For simplicity, focuses on the $k$-th arm at time $t$,

\begin{equation*}
\begin{aligned}
    \bQ := \bQ_{k,t-1} \mbox{, }
    \bX := \bX_{t-1} \mbox{, }
    \bY := \bY_{t-1} \mbox{, }
    \bepsilon := \bepsilon_{t-1} \mbox{, } 
    \bV := \bV_{t}
\end{aligned}
\end{equation*}
Therefore,

\begin{subequations}
\begin{align}
    RSS_{k, t} 
    =& \norm{\bQ (\bY - \bX \Hat{\btheta}_{t})}_2^2 \\
    =& \norm{\bQ (\bY - \bX \bV_t^{-1}\bX^{\top}\bY )}_2^2 \\
    =& \norm{\bQ (\bI - \bX \bV_t^{-1}\bX^{\top})\bY}_2^2 \\
    =& \norm{\bQ (\bI - \bX \bV_t^{-1}\bX^{\top})\bX\btheta +  \bQ (\bI - \bX \bV_t^{-1}\bX^{\top})\bepsilon}_2^2 
    \quad \text{(by $\bY = \bX\btheta + \bepsilon$)}\\
    =& 
    \norm{\bQ (\bI - \bX \bV_t^{-1}\bX^{\top})\bX\btheta}_2^2 +
    \norm{\bQ (\bI - \bX \bV_t^{-1}\bX^{\top})\bepsilon}_2^2 \notag \\
    &+
    2 \btheta^{\top} \bX^{\top} (\bI - \bX \bV_t^{-1}\bX^{\top}) \bQ^{\top} \bQ (\bI - \bX \bV_t^{-1}\bX^{\top}) \bepsilon
\end{align}
\end{subequations}
\end{proof}

\subsection{Proof of Lemma \ref{appendix: lemma_two_side_bound}}

\begin{proof}
\label{appendix: proof_lemma_two_side_bound}
Follow the same simplified notations in \ref{appendix: proof_lemma_RSS}, 
\begin{equation*}
\begin{aligned}
    \bQ := \bQ_{k,t-1} \mbox{, }
    \bX := \bX_{t-1} \mbox{, }
    \bY := \bY_{t-1} \mbox{, }
    \bepsilon := \bepsilon_{t-1} \mbox{, } 
    \bV := \bV_{t}
\end{aligned}
\end{equation*}
In the following part of proof, we overload the notations for singular value decomposition of matrices $\bX_{t-1}$ and $\bX_K$, note that this notations are only used in this proof for lemma \ref{appendix: lemma_two_side_bound}, 
\begin{equation*}
\begin{aligned}
    \bX := \bX_{t-1}=\bG \bSigma \bU \text{  and  }
    \bM := \bI - \bX \bV^{-1} \bX^{\top}
\end{aligned}
\end{equation*}
Further denote $s:=s_{1,t-1}$ and 
$$\ba := \frac{1}{\sqrt{s}}\bM \bQ^{\top} \bQ \bM \bX \btheta = (a_1,...,a_{t-1})^{\top} $$
\textbf{Step 1: Two sided bounds given $\ba$}\\
The key observation is that random vector $\ba$ is deterministic given history $\mF_{t-2} \cup \{\{\omega_{k,t-1,i}\}_{i=1}^{s_{k,t-1}}\}_{k=1}^{K}$. Recalling that noise $\epsilon_{\tau}$ is independent of $\omega_{k,t,i}$ for $\forall \tau, k, t, i$, by conditioning on $\mF_{t-2}$, 
\begin{equation}
\begin{aligned}
    \BE[e^{\eta \xi_{t}}] 
    = \BE[\BE[ e^{\eta \ba^{\top} \bepsilon} |\mF_{t-2} \cup \{\{\omega_{k,t-1,i}\}_{i=1}^{s_{k,t-1}}\}_{k=1}^{K}]]
     = \BE[e^{\eta \sum_{i=1}^{t-2} a_i \epsilon_i} \BE[ e^{a_{t-1} \epsilon_{t-1}} |\mF_{t-2}]]  
\end{aligned}
\end{equation}
which indicates
\begin{equation}
\begin{aligned}
    \BE[e^{\eta^2 \sum_{i=1}^{t-2} a_i \epsilon_i} \cdot e^{\eta^2 a_{t-1}^2 L_1}]
    \leq 
    \BE[e^{\eta \xi_{t}}]
    \leq 
    \BE[e^{\eta^2 \sum_{i=1}^{t-2} a_i \epsilon_i} \cdot e^{\eta^2 a_{t-1}^2 L_2}]
\end{aligned}
\end{equation}
Therefore, by conditioning on $\mF_{t-2},\mF_{t-3},...,\mF_{1}$ consecutively, the partial randomness from vector $\ba$ is left to integrated by the outside expectation $\BE$ and 
\begin{equation}
\begin{aligned}
    \BE[e^{\eta^2 \norm{\ba}_2^2 L_1}]
    \leq 
    \BE[e^{\eta \xi_{t}}]
    \leq 
    \BE[e^{\eta^2 \norm{\ba}_2^2 L_2}]
\end{aligned}
\end{equation}
\textbf{Step 2: Two sided bounds for $\norm{\ba}_2^2$}\\
Another key observation is from eigenvalues of $\bX \bV^{-1} \bX^{\top}$ under the ridge regression procedure. It can be shown that the eigenvalues of matrix $\bX \bV^{-1} \bX^{\top}$ are $\frac{\sigma_1^2}{\sigma_1^2 + \lambda},.., \frac{\sigma_r^2}{\sigma_r^2 + \lambda}$ and $t-1-r$ zeros. Thus, spectral decomposition of matrix $\bM$ is, $\bM = \bG (\bI - \bOmega) \bG^{\top}$ and $\bI - \bOmega$ is diagonal matrix with with diagonal elements $\frac{\lambda}{\sigma_1^2 + \lambda},.., \frac{\lambda}{\sigma_r^2 + \lambda}$ and $t-1-r$ ones. We use $\lambda_{\max}(\bA)$ to denote the maximum eigenvalue of a matrix $\bA$.\\
Thus, 

\begin{subequations}
\begin{align}
    \norm{\ba}_2^2 
    &= \frac{1}{s} \btheta^{\top} \bX^{\top} \bM \bQ^{\top} \bQ \bM \bM \bQ^{\top} \bQ \bM \bX \btheta \\
    & = \frac{1}{s} \btheta^{\top} \bX^{\top} \bG (\bI - \bOmega) \bG^{\top} \bQ \bG (\bI - \bOmega) \bG^{\top} \bG (\bI - \bOmega) \bG^{\top} \bQ \bG (\bI - \bOmega) \bG^{\top} \bX \btheta \\
    & = \frac{1}{s} \btheta^{\top} \bX^{\top} \bG (\bI - \bOmega) \bG^{\top} \bQ \bG (\bI - \bOmega)^{2} \bG^{\top} \bQ \bG (\bI - \bOmega) \bG^{\top} \bX \btheta
\end{align}
\end{subequations}
For upper bound,
\begin{subequations}
\begin{align}
    \norm{\ba}_2^2 
    & \leq \btheta^{\top} \bX^{\top} \bG (\bI - \bOmega) \bG^{\top} \bQ \bG (\bI - \bOmega)^{2} \bG^{\top} \bQ \bG (\bI - \bOmega) \bG^{\top} \bX \btheta
    \quad (s \geq 1)\\
    & \leq \lambda_{\max}((\bI - \bOmega)^{2}) \btheta^{\top} \bX^{\top} \bG (\bI - \bOmega) \bG^{\top} \bQ \bG (\bI - \bOmega) \bG^{\top} \bX \btheta  \\
    & =  \btheta^{\top} \bU^{\top} \bSigma^{\top} (\bI - \bOmega) \bG^{\top} \bQ \bG (\bI - \bOmega) \bSigma \bU \btheta 
    \quad (\bX := \bG \bSigma \bU \text{ and } \lambda_{\max}((\bI - \bOmega)^{2}) = 1)\\
    & \leq \btheta^{\top} \bU^{\top} \bSigma^{\top} (\bI - \bOmega)^{2} \bSigma \bU \btheta
    \quad (\lambda_{\max}(\bQ) = 1)\\
    & \leq 
    \btheta^{\top} \bU^{\top} \bSigma^{\top} \bSigma \bU \btheta\\
    & \leq 
    \sigma_{\max}^2 \btheta^{\top} \bU^{\top}  \bU \btheta
    \quad (\lambda_{\max}(\bSigma^{\top} \bSigma) = \sigma_{\max}^2)\\
    & = 
    \sigma_{\max}^2 \norm{\btheta}_2^2\\
    & \leq 
    \sigma_{\max}^2 S_2^2
\end{align}
\end{subequations}
For lower bound,
\begin{subequations}
\begin{align}
    \norm{\ba}_2^2 
    & \geq 
    \frac{1}{s}
    \lambda_{\min}((\bI - \bOmega)^{2}) \btheta^{\top} \bU^{\top} \bSigma^{\top} (\bI - \bOmega) \bG^{\top} \bQ \bG (\bI - \bOmega) \bSigma \bU \btheta \\
    & = \frac{1}{s} (\frac{\lambda}{\sigma_{\max}^2 + \lambda})^2
    \btheta^{\top} \bU^{\top} \bSigma^{\top} (\bI - \bOmega) \bG^{\top} \bQ \bG (\bI - \bOmega) \bSigma \bU \btheta 
    \quad (\lambda_{\min}((\bI - \bOmega)^{2}) = (\frac{\lambda}{\sigma_{\max}^2 + \lambda})^2) \\
    & = \frac{1}{s} (\frac{\lambda}{\sigma_{\max}^2 + \lambda})^2
    \btheta^{\top} (\bX - \bZ)^{\top} \bQ (\bX - \bZ) \btheta 
    \quad (\bZ:= \bG \bOmega \bSigma \bU) \\
    & =
    (\frac{\lambda}{\sigma_{\max}^2 + \lambda})^2 
    \btheta^{\top} (\bx_1 - \bz_1) (\bx_1 - \bz_1)^{\top} \btheta \\
    & = 
    (\frac{\lambda}{\sigma_{\max}^2 + \lambda})^2 
    ((\bx_1 - \bz_1)^{\top} \btheta)^2 \\
    & \geq 
    (\frac{\lambda}{\sigma_{\max}^2 + \lambda})^2 
    S_1^2
\end{align}
\end{subequations}
Therefore, $\forall \eta \geq 0$,
\begin{equation}
\begin{aligned}
    \exp(\frac{\lambda^2}{(\sigma_{max}^2 + \lambda)^2} 
    S_1^2 L_1 \eta^2)
    \leq 
    \BE[e^{\eta \xi_{t}}]
    \leq 
    \exp(\sigma_{\max}^2 S_2^2 L_2 \eta^2 )
\end{aligned}
\end{equation}
\end{proof}

\subsection{Proof of Lemma \ref{appendix: lemma_subG}}
\begin{proof}
\label{appendix: proof_lemma_subG}
This proof is inspired by the Theorem 1 and its proof in \citep{zhang2020non}. Also, an important lemma, lemma \ref{appendix: Paley-Zygmund}, which is called Paley-Zygmund inequality is used. Since $t=0$ is the trivial case, in the following part, we assume $t>0$. Take
\begin{equation}
\label{appendix: proof_lemma_subG_x}
\begin{aligned}
    x:=R_{1} t - \frac{1}{t} \quad \forall t > 0
\end{aligned}
\end{equation}
Then
\begin{subequations}
\begin{align}
    \BP(X \geq R_{1} t - \frac{1}{t})
    &= 
    \BP(e^{tX} \geq e^{R_{1}t^2 - 1}) \\
    & \geq 
    \BP(e^{tX} \geq e^{-1} \BE[e^{tX}]) \\
    & \geq
    (1 - e^{-1})^{2} \dfrac{(\BE[e^{tX}])^2}{\BE[e^{2tX}]} 
    \quad (\text{by lemma \ref{appendix: Paley-Zygmund}})\\
    & \geq 
    (1 - e^{-1})^{2} \dfrac{(e^{R_{1} t^{2}})^2}{e^{4R_{2} t^{2}}} \\
    & = 
    (1 - e^{-1})^{2} \exp(-(4R_{2} - 2R_{1})t^{2})
\end{align}
\end{subequations}
By (\ref{appendix: proof_lemma_subG_x}), $t$ satisfies a quadratic equation $R_{1} t^{2} - xt - 1 =0$. Since $t > 0$, 
\begin{equation}
\begin{aligned}
    t = \dfrac{x + \sqrt{x ^{2} + 4R_{1}}}{2R_{1}}
\end{aligned}
\end{equation}
Therefore,
\begin{subequations}
\begin{align}
    \BP(X \geq x) 
    & \geq 
    (1 - e^{-1})^{2} \exp(-(4R_{2} - 2R_{1})(\frac{x + \sqrt{x ^{2} + 4R_{1}}}{2R_{1}})^{2}) \\
    & = 
    (1 - e^{-1})^{2} \exp(-\frac{2R_{2} - R_{1}}{2R_{1}^{2}}(4x^2 + 8 R_{1})) \\
    & =
    (e-1)^{2} e^{\frac{8R_{2}}{R_{1}} - 6}
    \exp(-\frac{4 R_{2} - 2 R_{1}}{ R_{1}^2} x^{2}) 
\end{align}
\end{subequations}
\end{proof}

\clearpage

\section{Supporting Lemmas}\label{appendix: suppporting}

\subsection{Confidence Ellipsoid under Least Squared Estimation}
\begin{lemma}
\label{appendix: Ellipsoid}
Under assumptions \ref{ass:bound} and \ref{ass:noise_bound} and notations from (\ref{def: LSE}), $\forall \alpha > 0$, with probability at least $1-\alpha$, for all $t \geq 1$, $\btheta$ lies in the following confidence ellipsoid,
\begin{equation}
\begin{aligned}
    \mC_{t}:=\{\btheta \in \BR^d: \norm{\btheta - \Hat{\btheta}_t}_{\bV_t} 
    \leq
    L_2 \sqrt{d \log(\frac{1 + tL^{2}/\lambda}{\alpha})} + \lambda^{1/2} S_2 \}
\end{aligned}
\end{equation}
\end{lemma}

\subsection{Lower Bound of Gaussian Tail}
\begin{lemma}
\label{appendix: Gaussian_Tail}
Set $Z \sim N(0,1)$. Then, $\forall c >0 $
\begin{equation}
\begin{aligned}
\BP(Z \geq t) \geq 
\begin{cases}
\frac{b}{\sqrt{2 \pi}} \exp(-\frac{3}{2} t^2) & \text{if } t\geq b \\
\Phi(-c) & \text{if } 0 < t < b
\end{cases}
\end{aligned}
\end{equation}
\end{lemma}

\subsection{Self-normalized Bound for Martingales}
\begin{lemma}
\label{appendix: self_normalized}
Let $\{\mF_{t}\}_{t=0}^{\infty}$ be a filtration and $\{\epsilon_{t}\}_{t=0}^{\infty}$ be a real-valued stochastic process such that:\\
(i) $\epsilon_t$ is $\mF_{t}$-measurable\\
(ii) $\epsilon_t$ is conditionally subgaussian with constant $R$, that is, for some $R$ and $\forall t \geq 0$
\begin{equation*}
    \BE[e^{\lambda \epsilon_t}|\mF_{t-1}] \leq e^{\frac{\lambda^2 R}{2}} \quad \forall \lambda \in \BR
\end{equation*}
Let $\{X_{t}\}_{t=0}^{\infty}$ be a $\BR^d$-valued stochastic process such that $X_{t}$ is $\mF_{t-1}$-measurable and assume $\bV$ is $d$ by $d$ positive definite matrix. For any $t$, define
\begin{equation*}
    \Bar{\bV}_t = \bV + \sum_{s=1}^t X_s X_s^{\top} \quad S_t = \sum_{s=1}^t \epsilon_t X_s
\end{equation*}
Then for any $\delta >0 $ and any $t \geq 0$, with probability at least $1-\delta$, 
\begin{equation*}
    \norm{S_t}_{\Bar{\bV}_t^{-1}}^2 \leq 2 R \log (\dfrac{det(\Bar{\bV}_t)^{1/2} det(\bV)^{-1/2}}{\delta})
\end{equation*}
\end{lemma}

\subsection{Second Moment Bound for Subgaussian Random Variables}
\begin{lemma}
\label{appendix: Second_Moment_SubG}
Suppose random variable $X$ is subgaussian with constant $R$, that is, $\BE[e^{t X}] \leq e^{R t^2}\mbox{ } \forall t \in \BR$, then
\begin{equation}
    \BE[X^2] \leq 4R
\end{equation}
\end{lemma}

\subsection{Paley-Zygmund Inequality}
\begin{lemma}
\label{appendix: Paley-Zygmund}
Suppose $X$ be a random variable, then when $\forall \theta \in [0,1]$ and $\forall t \geq 0$,
\begin{equation}
\begin{aligned}
\BP(e^{tX} \geq \theta \BE[e^{tX}]) \geq (1 - \theta)_{+}^2 \dfrac{(\BE[e^{tX}])^2}{\BE[e^{2tX}]}
\end{aligned}
\end{equation}
\end{lemma}

\clearpage

\section{Supplement to Experiments}\label{app: experiments}

\subsection{Algorithms for \texttt{LinReBoot}}
\label{appendix: experiment_algs}

In the paper, Algorithm \ref{alg:LinReBoot: Version_1} implements \texttt{LinReBoot} for the stochastic bandit problems. In our experiments, there are two other additional setting with linear reward function for linear bandit problem. We provide other two implementations of \texttt{LinReBoot}. The first one is \texttt{LinReBoot} for linear contextualized bandit, which is given in Algorithm \ref{alg:LinReBoot: Version_2}. Another one is \texttt{LinReBoot} for linear bandit with covariates, which is given in Algorithm \ref{alg:LinReBoot: Version_3}. 

\begin{algorithm}[h!]
\caption{\texttt{LinReBoot} in Contextual Linear Bandit}\label{alg:LinReBoot: Version_2}
\begin{algorithmic}
\Require $\lambda$, $s_{1,0}=...=s_{K,0}=0$
\For{$t = 1,...,n$}
    \If{$t < K + 1$}
        \State $I_{t} \leftarrow t$
    \Else
        \State Get new contexts $\bx_1,...,\bx_K$ 
        \State $\bV_{t} \leftarrow \bX_{t-1}^{\top} \bX_{t-1} + \lambda \bI$
        \State $\Hat{\btheta}_{t} \leftarrow \bV_{t}^{-1} \bX_{t-1}^{\top} \bY_{t-1}$
        \For{$k=1,...,K$}
            \State $e_{k,t,i} \leftarrow r_{k,i} - \bx_{k}^{\top} \Hat{\btheta}_{t}$, $\forall i \in \{s_{k,t-1}\}$
            \State Generate $\{\omega_{k,t,i}\}_{i=1}^{s_{k,t-1}}$
            \State $\Tilde{\mu}_{k} \leftarrow \bx_{k}^{\top} \Hat{\btheta}_{t} +  
                    s_{k,t-1}^{-1}\sum_{i=1}^{s_{k,t-1}} \omega_{k,t,i} e_{k,t,i} $
        \EndFor
        \State $I_{t} \leftarrow \underset{k \in [K]}{\arg\max} \mbox{ }  \Tilde{\mu}_{k}$
    \EndIf
    \State $s_{I_{t},t} \leftarrow s_{I_{t},t-1} + 1$ and $s_{k,t} \leftarrow s_{k,t-1}$. $\forall k \neq I_{t}$
    \State Pull arm $I_{t}$ and get reward $r_{I_{t}, s_{I_{t}}}$
    \State 
    $\bX_{t} \leftarrow
    \begin{bmatrix}
    \bX_{t-1} \\
    \bx_{I_{t}}^{\top}
    \end{bmatrix}$
    and 
    $\bY_{t} \leftarrow
    \begin{bmatrix}
    \bY_{t-1} \\
    r_{I_{t}, s_{I_{t}}}
    \end{bmatrix}$
\EndFor
\end{algorithmic}
\end{algorithm}

\begin{algorithm}[h!]
\caption{\texttt{LinReBoot} in Linear Bandit wit Covariates}\label{alg:LinReBoot: Version_3}
\begin{algorithmic}
\Require $\lambda$, $s_{1,0}=...=s_{K,0}=0$
\For{$t = 1,...,n$}
    \If{$t < K + 1$}
        \State $I_{t} \leftarrow t$
    \Else
        \State Get new context $\bx_t$
        \For{$k=1,...,K$}
            \State $\bV_{k,t} \leftarrow \bX_{k,t-1}^{\top} \bX_{k,t-1} + \lambda \bI$
            \State $\Hat{\btheta}_{k,t} \leftarrow \bV_{k,t}^{-1} \bX_{k,t-1}^{\top} \bY_{k,t-1}$
            \State $e_{k,t,i} \leftarrow r_{k,i} - \bx_{k}^{\top} \Hat{\btheta}_{t}$, $\forall i \in \{s_{k,t-1}\}$
            \State Generate $\{\omega_{k,t,i}\}_{i=1}^{s_{k,t-1}}$
            \State $\Tilde{\mu}_{k} \leftarrow \bx_{k}^{\top} \Hat{\btheta}_{t} +  
                    s_{k,t-1}^{-1}\sum_{i=1}^{s_{k,t-1}} \omega_{k,t,i} e_{k,t,i} $
        \EndFor
        \State $I_{t} \leftarrow \underset{k \in [K]}{\arg\max} \mbox{ }  \Tilde{\mu}_{k}$
    \EndIf
    \State $s_{I_{t},t} \leftarrow s_{I_{t},t-1} + 1$ and $s_{k,t} \leftarrow s_{k,t-1}$. $\forall k \neq I_{t}$
    \State Pull arm $I_{t}$ and get reward $r_{I_{t}, s_{I_{t}}}$
    \State 
    $\bX_{I_t, t} \leftarrow
    \begin{bmatrix}
    \bX_{I_t, t-1} \\
    \bx_{t}^{\top}
    \end{bmatrix}$
    and 
    $\bY_{I_t, t} \leftarrow
    \begin{bmatrix}
    \bY_{I_t, t-1} \\
    r_{I_{t}, s_{I_{t}}}
    \end{bmatrix}$
\EndFor
\end{algorithmic}
\end{algorithm}

\subsection{Experimental Setting}
\label{appendix: experiment_settings}
This part provides the detailed description of the experimental setting in Section \ref{sec:exp_to_imple}. There are three settings in our experiment: Stochastic Linear Bandit, Contextual Linear Bandit and Linear Bandit with Covariates. Each of them has own synthetic data generation procedure which is described in the following parts.\\

\textbf{Stochastic Linear Bandit.}
In the first experiment, we compare \texttt{LinReBoot} to other linear bandit algorithms under stochastic linear bandit described in Section \ref{sec: Stochastic_Linear_Bandit}. The \texttt{LinReBoot} is implemented as the efficient version of algorithm \ref{alg:LinReBoot: Version_1}. Our experiment is conducted under three choice of dimension $d$ including $5$, $10$ and $20$. The number of arm in this setting is $100$. True parameter $\btheta$ has norm $1$ and is generated from uniform distribution by entries. In other word, generate $\theta_i \sim U(-0.5, 0.5), \forall i \in [d]$ and then shrink $\norm{\btheta}_2 = 1$. Context features $\bx_1,...,\bx_K$ are generated by $\bx_{ik} \sim U(0,1), \forall i \in [d],  k \in [K]$ and normalized to $\norm{\bx_k}_2 = 1$. By the normalization of $\btheta$ and $\{\bx_k\}_{k=1}^{K}$, the true mean of reward is bounded by $1$, making \texttt{LinPHE} and \texttt{LinGIRO} become easier to choose a reasonable bounds for reward. Noise $\epsilon_t$ is generated from $N(0, 0.1)$. At each choice of $d$, our results are averaged over 100 randomly chosen environment and we evaluate all algorithms under the exact same environment with horizon length $10000$. Regularization parameter $\lambda$ is chosen as $0.1$ through out the experiments. Tuning parameters for each algorithms are described in Appendix \ref{appendix: experiment_other_algs}.

\textbf{Contextual Linear Bandit.}
In the second experiment, we compare \texttt{LinReBoot} to other linear bandit algorithms under linear bandit with uncertain/random context.  We experiment with several dimensions $d$ including $5$, $10$ and $20$. The number of arm is $100$. True parameter is generated by the same way as stochastic linear bandit setting in Section \ref{subsec: SLB_experiment}. Contexts of arm $k$ has distribution $N_d(\bnu_k, 1/(2K) \bI)$ where $\bnu_{k}$ is generated by following: $\bnu_{ik} \sim U(0,1), \forall i \in [d] \quad k \in [K]$ and normalized to $\norm{\bnu_k}_2 = 1$. Note that $\bnu_k$ are predefined before the simulation. Noise $\epsilon_t$ is generated from $N(0, 0.5)$. Remaining environment setting is designed as the same in Section \ref{subsec: SLB_experiment}: number of simulation is $100$, horizon length is $10000$, regularization parameter $\lambda= 0.1$. Most hyperparameters are chosen as the same as Section \ref{subsec: SLB_experiment} except for the reward bounds in \texttt{LinPHE} and \texttt{LinGIRO}. Detailed description is provided in Appendix \ref{appendix: experiment_other_algs}.

\textbf{Linear Bandit with Covariates}
Our last experiment is conducted under the setting of linear bandit with covariates. Again, we experiment with several dimensions $d$ including $5$, $10$ and $20$ while the number of arms is $10$ in this setting. True parameter $\btheta_1,...,\btheta_K$ are generated one by one and each of them is generated in the following way: (1) Choose an integer $n_{-} \leq d$ by $n_{-} \sim Binomial(d, 1/2)$ and randomly sample $n_{-}$ integers from $1$ to $d$, these $n_{-}$ integers indicates the entries that has negative direction in $\btheta_k$. (2) generate a $d$-dimensional vector with $n_{-1}$ entries are $-1$ and remaining $n_{+}: = d - n_{-}$ entries are $1$ by the $n_{-}$ integers sampled in the previous step. (3) Each entries will add a random perturbation from $U(-0.95, 0.95)$ to make the magnitude of the each entry is spread between $0.05$ to $1$. (4) The resulting vector will be normalized by $\norm{\btheta_k} = \frac{k}{K}$, indicating the norm of the true parameters $\btheta_1,...,\btheta_K$ are designed as $\frac{1}{K}, ..., 1$. Contexts are sampled from $N(\bZeros, \bI)$ which is independent of arms. Noise $\epsilon_t$ is generated from $N(0, 0.1)$. Remaining environment setting is designed as the same in Section \ref{subsec: SLB_experiment} or Section \ref{subsec: LB_random_experiment}: number of repetition is $100$ and horizon length is $10000$ as well as $\lambda= 0.1$. Reward bounds in \texttt{LinPHE} and \texttt{LinGIRO} are chosen based on the noise variance and other algorithms are designed as the same as the previous two settings. More specific description is provided in Appendix \ref{appendix: experiment_other_algs}.

\subsection{\texttt{LinReBoot} in Stochastic Linear Bandit}
The algorithm of \texttt{LinReBoot} is described in Algorithm \ref{alg:LinReBoot: Version_1} and steps of our \texttt{LinReBoot} and its efficient implementation under Gaussian bootstrap weights are  summarized in Section \ref{sec: LinReBoot_alg}. For the parameter tuning of \texttt{LinReBoot}, our first step candidate set for $\sigma_{\omega}$ in \texttt{LinReBoot} is $\{0.05, 0.1, 0.2, 0.5, 1.0\}$. The following result, figure \ref{fig: SLB_tuning}, shows that the values $0.05$, $0.1$, $0.2$ are not enough for resampling exploration under all three choice of context dimension. However, we notice that too large $\sigma_{\omega}$ leads to slow convergence even if it is indeed sub-linear. Thus $0.5$ is the best result under our stochastic linear bandit setting. We decide to do the further fined tuning, using the candidate set $\{0.3, 0.4, 0.6, 0.7\}$ and the result is shown in figure \ref{fig: SLB_tuning_plus}. It is clear that $\sigma_{\omega} = 0.3$ is the best choice when $d=5$ while $\sigma_{\omega} = 0.4$ is the best choice under the setting of $d=10$. When $d=20$, we conclude that $\sigma_{\omega} = 0.5$ is better than other candidates. As a result, our experiment in Section \ref{sec:exp_to_imple} choose $\sigma_{\omega} = 0.3$ for $d=5$, choose $\sigma_{\omega} = 0.4$ for $d=10$ and choose $\sigma_{\omega} = 0.5$ for $d=20$.

\begin{figure}[ht]
\centering
\includegraphics[width=\linewidth, height=0.22\textheight]{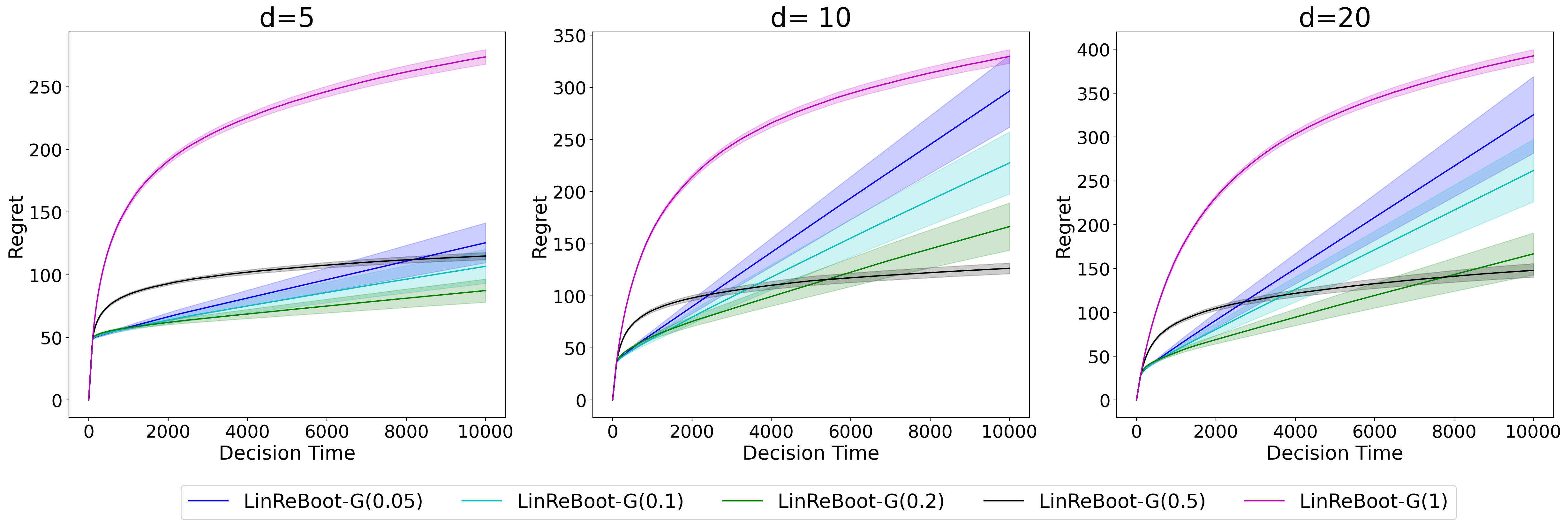}

\caption{First Step Tuning for \texttt{LinReBoot-G} under Stochastic Linear Bandit. The $x$ axis is round $t$ and $y$ axis is cumulative regret. The candidate set for $\sigma_{\omega}$ is $\{0.05, 0.1, 0.2, 0.5, 1.0\}$ and these three plots from left to right corresponds to $d=5$, $d=10$ and $d=20$ respectively.}
\label{fig: SLB_tuning}
\end{figure}

\begin{figure}[ht]
\centering
\includegraphics[width=1\linewidth, height=0.22\textheight]{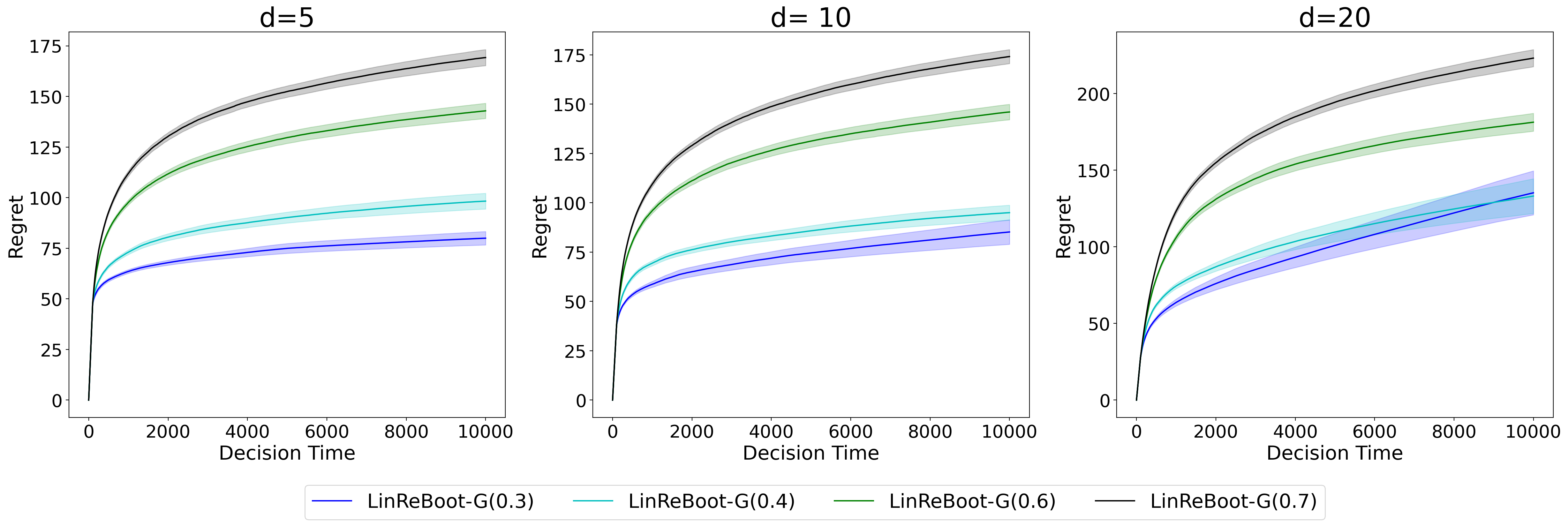}

\caption{Second Step Tuning for \texttt{LinReBoot-G} under Stochastic Linear Bandit. The $x$ axis is round $t$ and $y$ axis is cumulative regret. The candidate set for $\sigma_{\omega}$ is $\{0.3, 0.4, 0.6, 0.7\}$ and these three plots from left to right corresponds to $d=5$, $d=10$ and $d=20$ respectively.}
\label{fig: SLB_tuning_plus}
\end{figure}

\subsection{\texttt{LinReBoot} in Contextual Linear Bandit}
The algorithm \ref{alg:LinReBoot: Version_2} is \texttt{LinReBoot} under Contextual Linear Bandit. It is almost the same as algorithm \ref{alg:LinReBoot: Version_1} while the algorithm new requires the random contexts from each arm at each round $t$. For the parameter tuning of \texttt{LinReBoot}, our candidate set is designed as $\{0.05, 0.1, 0.2, 0.5, 1.0\}$ and the following result shows that $\sigma_{\omega} = 0.05$ is the best choice for all three design of context dimension $d$. Thus our experiment choose $\sigma_{\omega} = 0.05$ for three possible $d$ under this setting of Contextual Linear Bandit.

\begin{figure}[ht]
\centering
\includegraphics[width=1\linewidth, height=0.22\textheight]{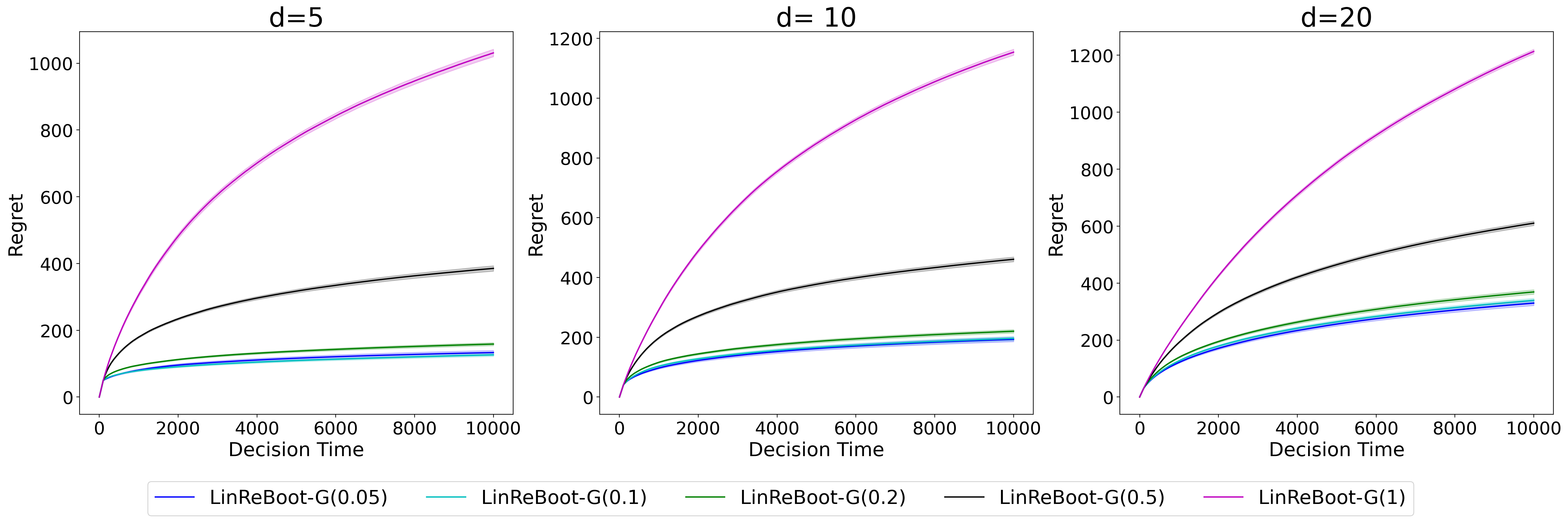}

\caption{Tuning for \texttt{LinReBoot-G} under Contextual Linear Bandit. The $x$ axis is round $t$ and $y$ axis is cumulative regret. The candidate set for $\sigma_{\omega}$ is $\{0.05, 0.1, 0.2, 0.5, 1.0\}$ and these three plots from left to right corresponds to $d=5$, $d=10$ and $d=20$ respectively.}
\label{fig: LB_random_tuning}
\end{figure}

\subsection{\texttt{LinReBoot} in Linear Bandit with Covariates}
The last version of \texttt{LinReBoot} is \texttt{LinReBoot} under Linear Bandit with Covariates which is provided as algorithm \ref{alg:LinReBoot: Version_3}. This algorithm is different from the previous two version due to the different task under linear bandit with covariates which requires the algorithm not only the estimation of the target parameter $\btheta$, but also detection of which arm a context belongs to. For the parameter tuning of \texttt{LinReBoot}, our candidate set is designed as $\{0.05, 0.1, 0.2, 0.5, 1.0\}$ and the following result shows that $\sigma_{\omega} = 1$ is the best choice for the cases including $d=5$ and $d=10$. When $d=20$, $\sigma_{\omega} = 1$ is still acceptable while $\sigma_{\omega} = 0.5$ might be preferred one. In fact, it must be pointed out that when $d$ becomes larger, the performances among difference choice of $\sigma_{\omega}$  becomes smaller and larger $\sigma_{\omega}$ might be worse for larger $d$. At the end, our experiment choose $\sigma_{\omega} = 1$ for $d = 5$ and $d=10$ and $\sigma_{\omega} = 0.5$ for $d = 20$.

\begin{figure}[ht]
\centering
\includegraphics[width=1\linewidth, height=0.22\textheight]{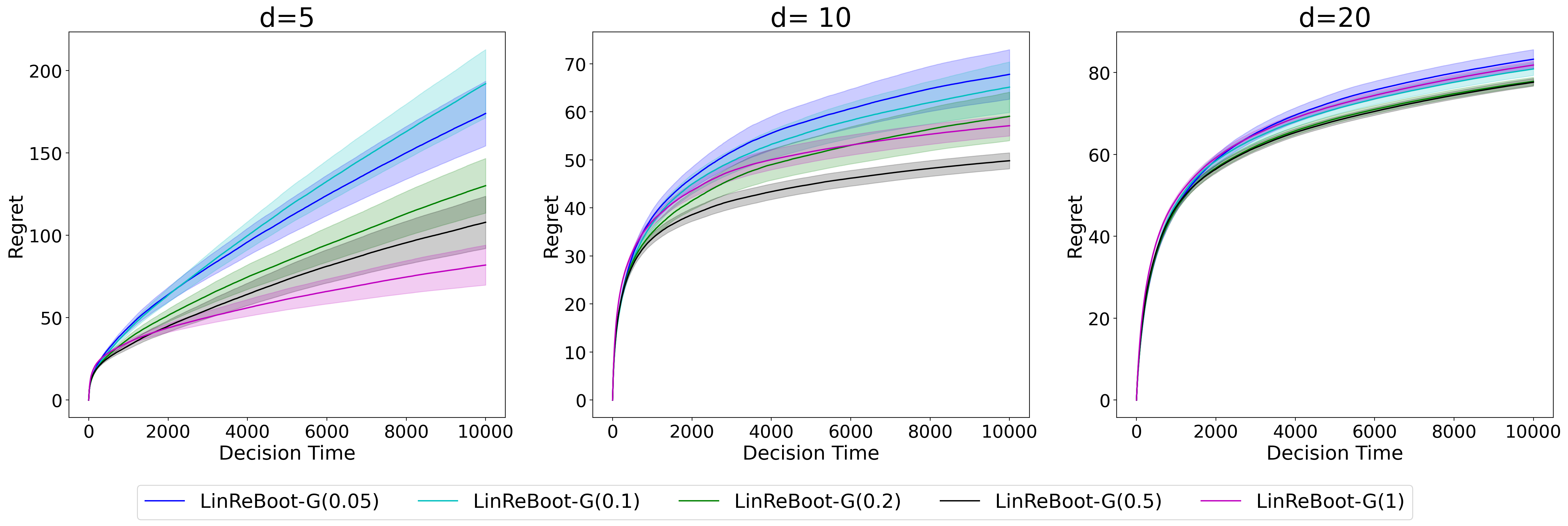}
\caption{Tuning for \texttt{LinReBoot-G} under Linear Bandit with Covariates. The $x$ axis is round $t$ and $y$ axis is cumulative regret. The candidate set for $\sigma_{\omega}$ is $\{0.05, 0.1, 0.2, 0.5, 1.0\}$ and these three plots from left to right corresponds to $d=5$, $d=10$ and $d=20$ respectively.}
\label{fig: LB_covariates_tuning}
\end{figure}

\subsection{Other Linear Bandit Algorithms}
\label{appendix: experiment_other_algs}
\textbf{Linear Thompson Sampling with Gaussian Prior} (\texttt{LinTS-G}). Thompson Sampling is a classic algorithm \citep{thompson1933likelihood} which requires only that one can sample from the posterior distribution over plausible problem instances (for example, values or rewards). Linear Thompson sampling is a Bayesian linear bandit algorithm which has studied by lots of previous works such as \citep{agrawal2013further, agrawal2013thompson, riquelme2018deep, russo2017tutorial}. In our experiment, we mainly depends on \citep{agrawal2013thompson, lattimore2020bandit} for implementing Linear Thompson sampling with Gaussian prior. There is almost the same among three different settings in our work. The only difference is that stochastic linear bandit and Contextual Linear Bandit is estimating/sampling parameter shared among arms while parameters are estimated/sampled using the rewards and contexts from only one arm in the setting of linear bandit with covariates. As mentioned in section \ref{sec:exp_to_imple}, the Gaussian prior variance is chosed as $\frac{1}{\lambda} = 10$ by Bayesian perspective of ridge regression model.

\textbf{Linear Thompson Sampling with Inverse Gamma Prior} (\texttt{LinTS-IG}). Another version of Thompson sampling under linear bandit is adding inverse gamma prior \citep{honda2014optimality, riquelme2018deep, bishop2006pattern}. We implement this inverse gamma version based on the detail suggested as \citep{riquelme2018deep}. Similar to \texttt{LinTS-G}, three settings share almost the same \texttt{LinTS-IG} and only difference is the parameters in linear bandit with covariates setting are estimated/sampled using the data from one arm. Moreover, Gaussian prior parameter is designed as $\frac{1}{10}$ which match our overall design for regularization $\lambda = 0.1$ and the inverse gamma prior parameters is suggest by \citep{riquelme2018deep}. More specifically, by $\sigma_0^2 \approx \alpha/(\alpha-1)$ where $\sigma_0^2 \tau^2 = 10$ is the initial variance on diagonal for sampling our target parameter $\btheta$, $\tau^2 = 5$ is Gaussian prior parameter and $\alpha = 2$ is the prior parameter for inverse gamma. 

\textbf{Linear Perturbed-History Exploration} (\texttt{LinPHE}). A well designed algorithm for stochastic linear bandit under bounded reward is \texttt{LinPHE} \citep{kveton2019perturbed}. The idea is also inspired from successfully adding exploration under Multi-armed bandit setting \citep{kveton2019perturbedMAB}. Our experiments use the suggested hyperparameter $a = 0.5$. However, since the original work is only designed for stochastic linear bandit with bounded rewards, we extended it to more general settings with Gaussian rewards. The detail is provided as follow. In stochastic linear bandit setting, based on our experimental design, true mean of each arm is bounded by $1$ and noise variance is set as $0.1$, indicating that we have high probability that the reward will be bounded by $1 + 3/\sqrt{10}$ on both sides. In the setting of Contextual Linear Bandit, the original efficient implementation from \citep{kveton2019perturbed} can not be used. But we modified by drawing a number from Binomial distribution $Binomial(\lceil a(t-1) \rceil, 1/2)$ at round $t$ and divided this number into $t-1$ parts randomly which are added as perturbation of rewards. The reward is bounded by $1 + 3/\sqrt{2}$. For the last setting, linear bandit with covariates, similar to previous setting, we modify by using Binomial distribution to adapt the non-integer value of $a$ but this time we need to apply the perturbed history by arm, that is using $Binomial(\lceil a s_{k, t-1} \rceil, 1/2)$ for all $k \in [K]$. The reward is bounded by $1.3$.

\textbf{Linear Garbage In Reward Out} (\texttt{LinGIRO}). Garbage In, Reward Out(\texttt{GIRO}) is a bootstrapping based algorithm designed for multi-armed bandit with bounded reward \citep{kveton2019garbage}. Since its idea of bootstrapping and perturbation on mean estimation is highly related to our residual bootstrapping exploration, it is worthy to compare with this classical bootstrapping based algorithm. But like \texttt{PHE}, it is originally designed for multi-armed bandit and we need to extend it to linear bandit setting with unbounded reward and then apply it to three settings in our experiment. Previous work \citep{kveton2019garbage, wang2020residual} suggest the conservative choice of $a$ is $1$, indicating adding one high pseudo reward and one low pseudo reward at each round. The detail, which is almost the same as previous modification for \texttt{LinPHE}, is provided as follow. In stochastic linear bandit and linear bandti with random context settings, we bootstrapping the previous reward-context pair and use the new sample to do least squared estimation. After pulling arm, $2a$ pseudo reward-context pairs are added: one is current context with reward upper bound and the other one is current context with reward lower bound. For the last setting, linear bandit with covariates, the only difference is that the bootstrapping is conducted by arm and the pseudo reward-context pairs are added to one arm at each round. The reward bound is chosen as $1 + 3/\sqrt{10}$ for stochastic linear bandit and $1 + 3/\sqrt{2}$ for the setting of Contextual Linear Bandit while $1.3$ is chosen for linear bandit with covariates setting.

\textbf{Linear Upper Confidence Bound} (\texttt{LinUCB}). Upper Confidence Bound(\texttt{UCB}) is a important type of bandit algorithms which is widely used. \texttt{LinUCB} is the version extended to linear bandit setting \citep{abbasi2011improved, chu2011contextual}. Since its popularity and usage, we believe it should be involved in our experiment and we implement \texttt{LinUCB} mainly relying on \citep{abbasi2011improved, lattimore2020bandit}. The confidence level is chosen as $95\%$ which matches the traditional statistical sense. Moreover, \texttt{LinUCB} is almost the same among three different setting. The only difference is stochastic linear bandit and Contextual Linear Bandit are using the rewards and contexts to estimate one target parameter, like \eqref{def: LSE} in our paper while the last setting, linear bandit with covariates, requires the least squared estimation to be done by arms. 

\subsection{Computation Efficiency}
\subsubsection{Efficient Implementation of \texttt{LinReBoot-G}}
\label{appendix: experiment_efficient}
Section \ref{subsec: Efficient_Implementation} discusses about why \texttt{LinReBoot-G} can be implemented efficiently. This section provides a further illustration and implementation in practice. First recall $\Tilde{\bmu}^{(t)} = (\Tilde{\mu}_{1,t},\dots, \Tilde{\mu}_{K,t})^\top$ is conditional distributed as
\begin{equation}
\begin{aligned}
    \Tilde{\bmu}^{(t)} | \mF_{t-1} \sim N_{K}(\Hat{\bmu}^{(t)}, \bSigma_{\omega}^{(t)})
\end{aligned}
\end{equation}
where $\Hat{\bmu}^{(t)} = (\Hat{\mu}_{1,t}, \dots, \Hat{\mu}_{K,t})^\top = \bX_{K} \Hat{\btheta}_t$ and $\bSigma_{\omega}^{(t)}$ is a diagonal matrix with diagonal elements $\sigma_{\omega}^2 s_{k,t-1}^{-2}RSS_{k,t}$. Note that $\bSigma^{(t)}$ can be computed by $\Hat{\bmu}^{(t)}$ and vectors,
\begin{align*}
    \br_1^{(t)} &:= (\sum_{i=1}^{s_{1,t-1}} r_{1,i}, \dots, \sum_{i=1}^{s_{K,t-1}} r_{K,i})^\top, \\
    \br_2^{(t)} &:= (\sum_{i=1}^{s_{1,t-1}} r_{1,i}^2, \dots, \sum_{i=1}^{s_{K,t-1}} r_{K,i}^2)^\top, \\
    \bs^{(t)} &:= (s_{1, t-1}, \dots, s_{K,t-1})^\top .
\end{align*}
These vectors can be updated incrementally by the above illustration. To sum up, when bootstrap weights are Gaussian, the efficient implementation for computing $\Tilde{\mu}_{k,t}$ at round $t$ has steps as follow, 

\begin{itemize}[leftmargin=15pt, itemsep = -2pt]
    \item Compute $\bV_{t}$, $\Hat{\btheta}_{t}$ and $\Hat{\bmu}^{(t)} = \bX_{K} \Hat{\btheta}_t$
    \item Compute $\bSigma^{(t)}$ using $\Hat{\bmu}^{(t)}$, $\br_1^{(t)}$, $\br_2^{(t)}$ and $\bs^{(t)}$
    \item Sample $\Tilde{\bmu}^{(t)} \sim N_{K}(\Hat{\bmu}^{(t)}, \bSigma^{(t)})$
    \item Pull arm $I_t$ and get its corresponding reward $r_{I_t}$
    \item Update $\br_1^{(t+1)}$, $\br_2^{(t+1)}$ and $\bs^{(t+1)}$
\end{itemize}

\subsubsection{Computational Cost}
\label{appendix: experiment_comp_cost}
The computation cost of linear bandit algorithms involved in our experiment are listed in the following table. Each running time is for one horizon with length $10000$. The settings are also provided in Appendix \ref{appendix: experiment_settings} and the description of algorithms are provided in Appendix \ref{appendix: experiment_other_algs}. 

\begin{table}[h]
\centering
\begin{tabular}{ll | lllllll}
\hline
\multicolumn{2}{c|}{\text{ Model }} & \multicolumn{6}{c}{\text { Run time (seconds) }} \\
 Setting  &    d      
 &  \texttt{LinReBoot} & \texttt{LinTS-G} 
 &  \texttt{LinTS-IG}  & \texttt{LinGIRO}
 &  \texttt{LinPHE}    & \texttt{LinUCB}  \\
\hline
  Stochastic Linear Bandit               &     5     &  3.2 & 1.8  & 2.2  & 6.5  & 4.0 & 6.2  \\
  Stochastic Linear Bandit               &     10    &  3.5 & 2.1  & 2.5  & 10.3 & 4.7 & 6.6  \\
  Stochastic Linear Bandit               &     20    &  4.8 & 3.9  & 3.8  & 24.6 & 5.6 & 7.4  \\
  Contextualized Linear Bandit           &     5     &  3.3 & 1.8  & 2.2  & 6.5  & 4.0 & 6.3  \\
  Contextualized Linear Bandit           &     10    &  3.5 & 2.1  & 2.5  & 10.2 & 4.7 & 6.6  \\
  Contextualized Linear Bandit           &     20    &  3.8 & 3.1  & 3.6  & 24.1 & 5.2 & 6.9  \\
  Linear Bandit with Covariates          &      5    &  1.4 & 7.8  & 12.9 & 10.3 & 5.2 & 1.2  \\
  Linear Bandit with Covariates          &     10    &  1.5 & 9.4  & 14.1 & 11.5 & 5.9 & 1.4  \\
  Linear Bandit with Covariates          &     20    &  1.6 & 14.2 & 18.9 & 15.2 & 7.4 & 1.5  \\
\hline
\end{tabular}

\caption{\footnotesize Computational Cost for Linear Bandit Algorithms}
\label{table: comp_cost}
\end{table}

\clearpage

\end{document}